\documentclass{article}

 \usepackage[preprint]{neurips_2026}

\usepackage[utf8]{inputenc} 
\usepackage[T1]{fontenc}    
\usepackage{hyperref}       
\usepackage{url}            
\usepackage{booktabs}       
\usepackage{amsfonts}       
\usepackage{nicefrac}       
\usepackage{microtype}      
\usepackage{xcolor}         

\usepackage{amsmath}
\usepackage{amssymb}
\usepackage{mathtools}
\usepackage{amsthm}
\newtheorem{lemma}{Lemma}
\newtheorem{proposition}{Proposition}
\usepackage{multirow}
\usepackage{graphicx}
\usepackage{subcaption}
\usepackage{algorithm}
\usepackage{algorithmic}
\usepackage{wrapfig}
\title{Towards Adaptive Continual Model Merging via Manifold-Aware Expert Evolution}

\author{%
  Haiyun Qiu$^1$, Xingyu Wu$^{1,*}$, Kay Chen Tan$^1$ \\
  \\
  $^1$Department of Data Science and Artificial Intelligence, \\
  The Hong Kong Polytechnic University, Hong Kong, China \\
  \texttt{haiyun.qiu@connect.polyu.hk}, \texttt{xingy.wu@polyu.edu.hk}, \texttt{kaychen.tan@polyu.edu.hk}
}

\begin{document}

\maketitle

\begin{abstract}
     Continual Model Merging (CMM) sequentially integrates task-specific models into a unified architecture without intensive retraining. However, existing CMM methods are hindered by a fundamental saturation-redundancy dilemma: backbone-centric approaches face parameter saturation and representation interference within fixed capacities, whereas Mixture-of-Experts (MoE) variants resort to indiscriminate expansion, incurring expert redundancy and a routing bottleneck reliant on additional data-driven optimization. To resolve these challenges, we propose MADE-IT (\underline{M}anifold-\underline{A}ware \underline{D}ynamic Expert \underline{E}volution and \underline{I}mplicit rou\underline{T}ing), an adaptive CMM method that orchestrates expert management and activation by grounding intrinsic expert representations in manifold geometry. We introduce a projection-based subspace affinity metric coupled with a distribution-aware adaptive threshold mechanism to guide autonomous expert evolution, harmonizing diversity with architectural parsimony. Furthermore, to bypass parameterized gating networks, we design a data-free and training-free implicit routing mechanism that activates experts via feature-subspace alignment. Extensive experiments demonstrate that MADE-IT consistently outperforms strong baselines in accuracy and robustness across long-horizon and shuffled task sequences, while significantly pruning redundant experts, particularly within generic modules and early layers.
\end{abstract}

\section{Introduction}

While fine-tuning pre-trained models is standard for task adaptation \citep{wolf2019huggingface,dodge2020fine}, the proliferation of independent checkpoints imposes prohibitive storage and deployment overheads \citep{mcmahan2017communication,huang2024emr}. Model merging offers a training-free alternative by directly integrating multiple task-specific models into a unified, versatile model \citep{yang2026model,yadavsurvey,li2025deep,ilharco2023editing,yadav2023ties,du2024parameter,zhou2024hm3}. However, conventional merging methods assume simultaneous model availability, rendering them ineffective in streaming scenarios where tasks and models arrive sequentially. Consequently, continual model merging (CMM) has emerged to enable sequential integration under a bounded memory footprint \citep{tang2025merging,qiu2025mingle,yang2025continual}.

Existing CMM methods face a fundamental saturation-redundancy dilemma. On one hand, integrating task-specific parameters into a fixed-capacity backbone via weighted averaging or orthogonal projection inevitably triggers parameter saturation and catastrophic forgetting \citep{tang2025merging,yang2025continual,qiu2025null}. Conversely, Mixture-of-Experts (MoE) variants mitigate task interference by allocating dedicated experts for each task \citep{cai2025survey,qiu2025mingle}. However, such indiscriminate expansion overlooks inter-expert commonalities, causing expert redundancy and architectural bloat. Resolving this dilemma requires adaptively retaining critical experts while pruning redundant ones during merging. Nevertheless, quantifying intrinsic inter-expert correlations to orchestrate such autonomous evolution remains highly non-trivial.

Beyond architectural challenges, MoE-based merging methods are further constrained by an explicit routing bottleneck during inference. Typically, they rely on parameterized gating networks to activate experts \citep{lu2024twin,shen2026efficient,qiu2025mingle}. However, these auxiliary networks inevitably introduce extra learnable weights and necessitate data-driven optimization to sustain routing accuracy. In the context of continual merging, frequent network tuning and inherent data dependency restrict their adaptability in strictly data-free and training-free scenarios. Consequently, a pressing objective is to propose an expert activation mechanism that efficiently captures the intrinsic alignment between input features and experts, bypassing the need for external routing networks.

Resolving these intertwined challenges necessitates extracting intrinsic expert representations to quantify correlations among experts and their alignment with input features. However, parameter-space metrics such as cosine similarity are inadequate due to the inherent permutation invariance and basis rotation sensitivity of deep neural networks \citep{entezari2022role,ainsworth2023git,kim2023warping}. Functionally equivalent experts may manifest as completely disparate parameter configurations, rendering redundancy identification and alignment evaluation intractable. Therefore, we shift from raw parameters to their underlying subspace structures. Grounded in manifold geometry \citep{edelman1998geometry,kornblith2019similarity}, we extract transformation-invariant expert representations. This geometric perspective bridges two critical operations: measuring inter-expert similarities for dynamic evolution and evaluating feature-subspace alignment for adaptive activation.  

Synthesizing these geometric insights, we propose MADE-IT (\underline{M}anifold-\underline{A}ware \underline{D}ynamic Expert \underline{E}volution and \underline{I}mplicit rou\underline{T}ing) for efficient expert management and activation in CMM. To address the saturation-redundancy dilemma, MADE-IT introduces a manifold-aware dynamic expert evolution strategy that considers both expert diversity and architectural parsimony. We extract expert principal subspaces via singular value decomposition (SVD) and characterize them as unique points on the Grassmann manifold. We then define a projection-based subspace affinity metric to quantify geometric similarities between candidate and existing experts. Guided by a distribution-aware threshold adaptation mechanism, MADE-IT autonomously orchestrates the expert population by deciding whether to create new experts for specificity or consolidate redundant ones via subspace merging. Furthermore, to address the routing bottleneck, we design a data-free, training-free implicit routing mechanism. This mechanism determines the optimal activation path by evaluating the projection alignment of input features onto expert principal subspaces while leveraging the hierarchical dependencies among modular experts. Our main contributions are summarized as follows:

\begin{itemize}
    \item We reveal the inadequacies of parameter-space metrics in quantifying the geometric correlations underlying expert redundancy and feature matching. In response, we introduce manifold geometry as a robust foundation for characterizing intrinsic expert representations.
    \item We propose a manifold-aware dynamic expert evolution strategy that enables the adaptive, differentiated management of modular expert populations based on subspace affinity, achieving a dynamic balance between expert diversity and architectural parsimony.
    \item We design a data-free and training-free implicit routing mechanism that exploits the geometric alignment between input features and expert principal subspaces, facilitating efficient expert activation without auxiliary gating networks.
    \item Extensive experiments demonstrate that MADE-IT achieves state-of-the-art performance, consistently outperforming strong baselines in accuracy and robustness while significantly reducing expert redundancy, particularly in generic modules and early layers.
\end{itemize}

\section{Background and Motivation}

\subsection{Problem Definition}
Consider a pre-trained backbone parameterized by $\theta^{(0)}$ and a sequence of task-specific models $\{\theta^{(t)}\}_{t=1}^{T}$ fine-tuned on associated datasets $\{\mathcal{D}_t\}_{t=1}^{T}$. The task vector for task $t$ is defined as $\tau^{(t)} = \theta^{(t)}-\theta^{(0)}$ \citep{ilharco2023editing}. Conventional merging performs one-shot merging of all models via $\theta_{\mathrm{merged}} = \mathcal{F}(\theta^{(0)}; \theta^{(1)}, \dots, \theta^{(T)})$, where $\mathcal{F}$ denotes a merging operator. In contrast, we focus on the continual merging scenario where task-specific models arrive sequentially. At step $t$, we integrate the incoming model $\theta^{(t)}$ with the current merged model $\theta_{\mathrm{merged}}^{(t-1)}$ to construct an updated model:
\begin{equation}
    \theta_{\mathrm{merged}}^{(t)}=\mathcal{F}(\theta^{(0)};\theta_{\mathrm{merged}}^{(t-1)},\theta^{(t)}),\mathrm{~}t\geq2.
\end{equation}

Typically, $\theta_{\mathrm{merged}}^{(1)}$ is initialized as the first arriving model $\theta^{(1)}$. Our objective is to sequentially merge incoming models in a data-free and training-free manner, while mitigating catastrophic forgetting.

\subsection{Related Work}
\noindent\textbf{Model Merging.}
Conventional model merging has evolved from simple weight averaging \citep{wortsman2022model,ainsworth2023git} to sophisticated techniques involving importance-based reweighting \citep{matena2022merging,jindataless}, pre-merging sparsification \citep{ilharco2023editing,yadav2023ties,du2024parameter}, and dynamic integration \citep{yangadamerging,lu2024twin,huang2024emr,shen2026efficient}. However, these approaches rely on simultaneous model availability and fail in streaming environments. Consequently, continual model merging has emerged to integrate sequentially arriving models while mitigating catastrophic forgetting \citep{tang2025merging,qiu2025mingle,yang2025continual,qiu2025null}. Pioneering CMM approaches employ orthogonal projection to reduce interference but inevitably dilute task-specific knowledge over time due to capacity constraints \citep{tang2025merging,yang2025continual}. While the recent MoE-based variant MINGLE \citep{qiu2025mingle} alleviates this bottleneck through dedicated expert allocation, it introduces redundant experts and relies on data-driven optimization of parameterized gating networks.

\noindent\textbf{Continual Learning.}
Continual learning aims to mitigate catastrophic forgetting in sequential task learning \citep{mccloskey1989catastrophic} through memory replay \citep{wang2022dualprompt,smith2023coda}, regularization constraints \citep{jung2020continual,wu2024meta}, and dynamic architectural expansion \citep{zhou2024expandable,marouf2024weighted}. However, these approaches demand frequent access to raw data and computationally intensive joint training. In contrast, CMM offers an efficient parameter-centric paradigm \citep{tang2025merging,qiu2025mingle}. By sequentially merging incoming task-specific models, CMM inherently circumvents the privacy risks associated with data sharing and significantly enhances the scalability and plasticity of foundation architectures.

\subsection{Challenges and Motivation}

\noindent\textbf{The Saturation-Redundancy Dilemma.}
Existing CMM methods either suffer from capacity saturation within a bounded parameter space or incur redundancy through unconstrained expert expansion. For instance, the backbone-centric approach OPCM \citep{tang2025merging} incrementally integrates task vectors $\tau^{(t)}$ into a fixed-capacity backbone by projecting them onto the orthogonal complement of historical updates via a projection mapping $\mathcal{P}_{\alpha}^{(t-1)}$ and a scaling factor $\lambda^{(t)}$:
\begin{equation}
\label{OPCM}
    \theta_{\mathrm{merged}}^{(t)}=\theta^{(0)} + \frac{\lambda^{(t-1)}\tau_{\mathrm{merged}}^{(t-1)}+\mathcal{P}_{\alpha}^{(t-1)}\tau^{(t)}}{\lambda^{(t)}}.
\end{equation}
While ensuring architectural consistency, backbone-centric approaches inevitably trigger parameter saturation and representation conflicts, exacerbating catastrophic forgetting \citep{yang2025continual,qiu2025null}. Conversely, the MoE-based approach MINGLE \citep{qiu2025mingle} circumvents interference by assigning a dedicated expert $f_i$ and an input-dependent gate $g_i$ to each sequential task. For a given input $X$, the output is formulated as:
\begin{equation}
\label{MINGLE}
    \theta^{(t)}_{\mathrm{merged}}(X)=\theta^{(0)}(X)+\sum_{i=1}^t{g_i(X)\cdot f_i(X)}.
\end{equation}
While avoiding interference, such unconditional expansion overlooks intrinsic commonalities among experts. Consequently, the expert population scales linearly with the model stream, yielding significant expert redundancy and architectural bloat. To solve this dilemma, we develop an adaptive, dynamic expert evolution strategy that harmonizes expert diversity with architectural parsimony. 

\noindent\textbf{The Bottleneck of Explicit Routing.}
MoE-based merging is further constrained by an explicit routing bottleneck. As defined in Eq.~\ref{MINGLE}, expert contributions are modulated by parameterized gating networks, typically formulated as linear projections with trainable parameters $W_t^{g}$ and $b_t^{g}$:
\begin{equation}
\label{gating}
    g_t(X) = W_t^{g\top}X + b_t^{g}.
\end{equation}
Crucially, these gating networks require continuous data access and frequent network tuning to ensure routing accuracy, limiting the applicability and flexibility of MoE-based CMM methods in strictly data-free and training-free scenarios. To bypass this bottleneck, we design an implicit routing mechanism for efficient expert activation without resorting to auxiliary parameterized networks.

\section{Method}

Motivated by the above, we propose MADE-IT for adaptive CMM. As illustrated in Fig.~\ref{fig:main-framework}, we first resolve the saturation-redundancy dilemma via a manifold-aware dynamic expert evolution strategy (Section~\ref{sec:Expert Evolution}). By representing experts via their principal subspaces on the Grassmann manifold, this strategy utilizes geometric affinity to automatically consolidate redundant experts while isolating specialized ones. Second, to circumvent the routing bottleneck, we design an implicit routing mechanism (Section~\ref{sec:Implicit Routing}). By exploiting geometric feature-subspace alignment, this mechanism facilitates efficient expert activation, bypassing the need for parameterized routing networks. Due to page limitations, the detailed algorithm of MADE-IT is outlined in Appendix \ref{ap:algorithm}.

\begin{figure*}[htbp]
    \centering
    \includegraphics[width=\textwidth]{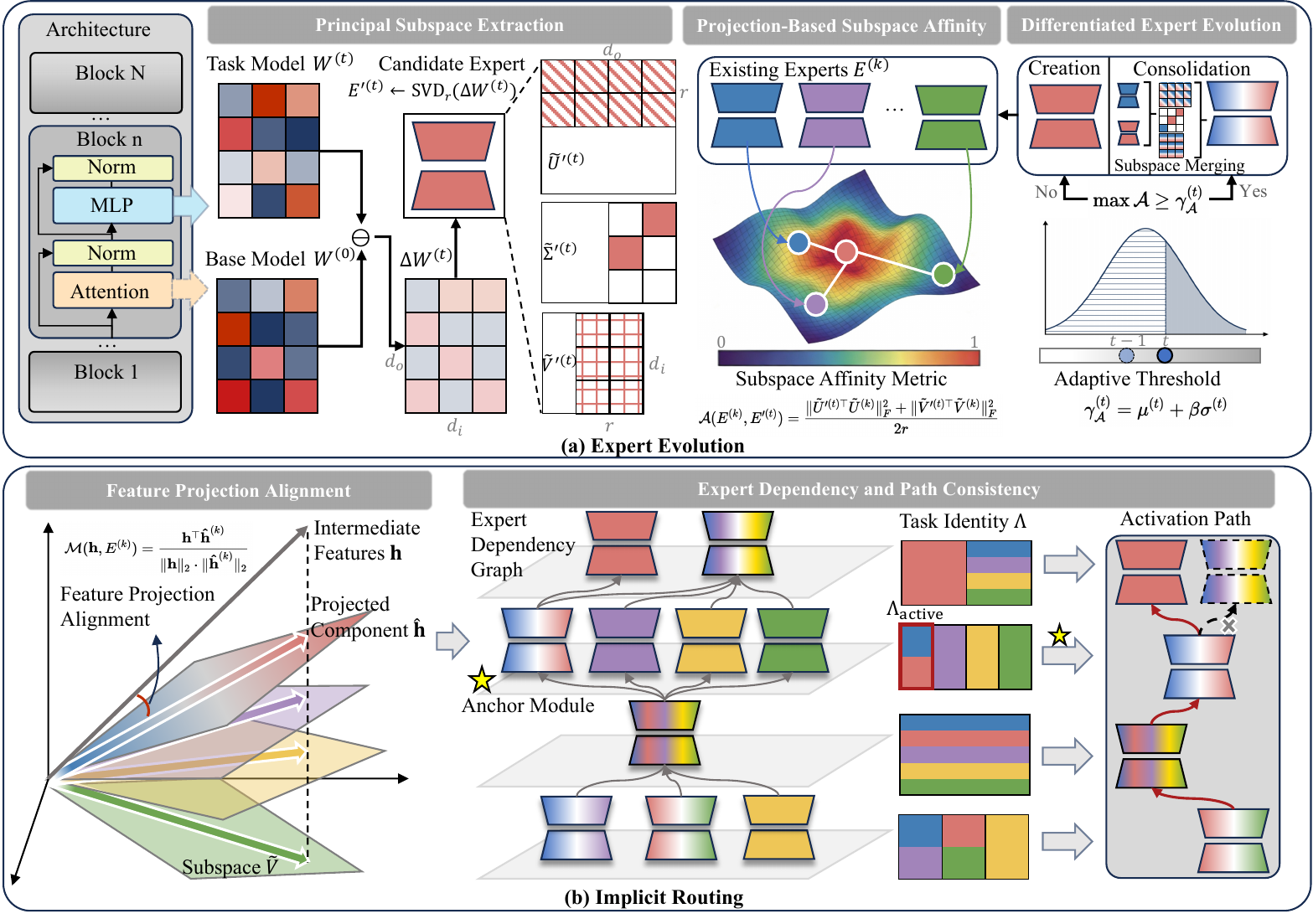}
    \caption{Overview of MADE-IT. (a) Expert Evolution: Expert principal subspaces are extracted from the module-wise weight updates via truncated SVD. Projection-based subspace affinity and adaptive thresholds guide the autonomous consolidation or creation of modular experts. (b) Implicit Routing: During inference, expert responsiveness is evaluated using intermediate features via feature projection alignment. Anchored by the highest-diversity module, the task-identity constraint is propagated across the expert dependency graph to resolve bifurcations and navigate a valid activation path.}
    \label{fig:main-framework}
\end{figure*}

\subsection{Manifold-Aware Dynamic Expert Evolution}
\label{sec:Expert Evolution}

Grounded in the low-rank nature of fine-tuned model updates \citep{hulora,gargiulo2025task}, our dynamic expert evolution strategy isolates essential knowledge increments within principal subspaces. Geometrically, characterizing these subspaces as unique points on the Grassmann manifold \citep{edelman1998geometry,hamm2008grassmann} yields compact, high-fidelity expert representations.

\subsubsection{Principal Subspace Extraction}
Given an incoming fine-tuned model for task $t$, let $\Delta W_c^{(t)} = W_c^{(t)} - W_c^{(0)} \in \mathbb{R}^{d_{o} \times d_{i}}$ denote the weight update of module $c$ relative to its pre-trained counterpart. Leveraging the low-rank prior of $\Delta W_c^{(t)}$, we apply truncated SVD to extract its principal directions. Specifically, we set a truncation rank $r = \lfloor \rho \cdot \min(d_{o}, d_{i}) \rfloor$ controlled by a ratio $\rho \in (0,1]$, and construct a lightweight modular expert via a rank-$r$ approximation:

\begin{equation}
\label{equation:E}
    E_c^{(t)} = \tilde{U}_c^{(t)} \tilde{\Sigma}_c^{(t)} \tilde{V}_c^{(t) \top},
\end{equation}
where $\tilde{U}_c^{(t)} \in \mathbb{R}^{d_{o} \times r}$ and $\tilde{V}_c^{(t)} \in \mathbb{R}^{d_{i} \times r}$ are the top-$r$ left and right singular vectors, while $\tilde{\Sigma}_c^{(t)} \in \mathbb{R}^{r \times r}$ contains the corresponding singular values. The columns of $\tilde{U}_c^{(t)}$ and $\tilde{V}_c^{(t)}$ form orthonormal bases for the principal output subspace $\mathcal{U}_c^{(t)}$ and input subspace $\mathcal{V}_c^{(t)}$, respectively. Geometrically, these subspaces correspond to unique points on the Grassmann manifolds $\mathcal{G}(r,d_o)$ and $\mathcal{G}(r,d_i)$. This decomposition effectively projects expert knowledge from the ambient parameter space onto these compact manifolds, enabling efficient storage and manipulation.

\subsubsection{Projection-Based Subspace Affinity}
As experts may vary in specificity and commonality across modules with different depths and types, we maintain an independently evolving expert set for each module $c$, denoted by $\mathbf{E}_c=\{E_c^{(k)}\}_{k=1}^{|\mathbf{E}_c|}$. To manage the expert population, we design a projection-based subspace affinity metric that is intrinsic to subspace geometry and invariant to basis selection.

\noindent\textbf{Projection Operators and Affinity Metric.}
For a reference expert $E_{c}^{(k)} \in \mathbf{E}_{c}$ and a new expert $E_c^{\prime(t)}$, let $\tilde{S}_c^{(k)}$ and $\tilde{S}_c^{\prime(t)}$ denote orthonormal bases for a target subspace $\mathcal{S} \in \{\mathcal{U}, \mathcal{V}\}$. The associated orthogonal projection operators are $\mathbf{P}_{\mathcal{S}^{(k)}} = \tilde{S}_c^{(k)}\tilde{S}_c^{(k)\top}$ and $\mathbf{P}_{\mathcal{S}^{\prime(t)}} = \tilde{S}_c^{\prime(t)}\tilde{S}_c^{\prime(t)\top}$, which uniquely characterize these subspaces and remain invariant to orthogonal rotations of the basis.

Then, we quantify the correlation between subspaces using the Hilbert-Schmidt inner product of their projection operators, $\langle \mathbf{P}_{\mathcal{S}^{(k)}}, \mathbf{P}_{\mathcal{S}^{\prime(t)}} \rangle_{HS} = \text{Tr}(\mathbf{P}_{\mathcal{S}^{(k)}} \mathbf{P}_{\mathcal{S}^{\prime(t)}})$, which measures subspace overlap. To ensure computational efficiency, we leverage the trace-Frobenius equivalence, where the Hilbert-Schmidt inner product of two projection operators is equivalent to the squared Frobenius norm of the inner product between their compact basis matrices (see Appendix \ref{ap:theory} for detailed derivation):
\begin{equation}
    \text{Tr}(\mathbf{P}_{\mathcal{S}^{(k)}} \mathbf{P}_{\mathcal{S}^{\prime(t)}}) = \| \tilde{S}_c^{\prime(t)\top} \tilde{S}_c^{(k)} \|_F^2.
\end{equation}
This equivalence obviates the explicit construction of the dense $d \times d$ projection matrices, reducing the computational cost to an $r \times r$ interaction. Finally, we normalize and aggregate the affinities of input and output subspaces to derive the final metric $\mathcal{A} \in [0, 1]$:
\begin{equation}
\label{equation:A}
    \mathcal{A}(E_{c}^{(k)}, E_c^{\prime(t)}) = \frac{\| \tilde{U}_c^{\prime(t)\top} \tilde{U}_c^{(k)} \|_F^2 + \| \tilde{V}_c^{\prime(t)\top} \tilde{V}_c^{(k)} \|_F^2}{2r}.
\end{equation}
Here, $\mathcal{A}=0$ denotes maximal subspace orthogonality, reflecting high specificity of the new candidate expert, while $\mathcal{A}=1$ indicates perfect alignment and full expert redundancy.

\noindent\textbf{Geometric Interpretation.}
We further analyze the geometric properties of the normalized subspace affinity for a single subspace type, defined as $\mathcal{A}_\mathcal{S} = \frac{1}{r} \| \tilde{S}_1^\top \tilde{S}_2 \|_F^2$, where $\tilde{S}_1$ and $\tilde{S}_2$ represent the corresponding orthonormal bases. Spectrally, the relative geometry of two Grassmannian subspaces is characterized by their principal angles $\{\phi_i\}_{i=1}^r$. Since the singular values of the interaction matrix $\tilde{S}_1^\top \tilde{S}_2$ precisely yield $\{\cos \phi_i\}_{i=1}^r$, our metric quantifies the spectral mean of the squared cosines of these principal angles, formulated as $\mathcal{A}_\mathcal{S} = \frac{1}{r} \sum_{i=1}^r \cos^2 \phi_i$. Crucially, $\mathcal{A}_\mathcal{S}$ efficiently aggregates alignment information across all principal directions via the Frobenius norm, bypassing the additional SVD calculation required for the full angle spectrum. Furthermore, we establish a linear mapping between $\mathcal{A}_\mathcal{S}$ and the standard Chordal distance $d_C=\frac{1}{\sqrt{2}}\|\tilde{S}_1 \tilde{S}_1^\top-\tilde{S}_2 \tilde{S}_2^\top\|_F$ on the manifold, yielding $\mathcal{A}_\mathcal{S}=1-\frac{1}{r}d_C^2$. Unlike the raw chordal distance $d_C \in [0, \sqrt{r}]$ which scales with the rank $r$, $\mathcal{A}_\mathcal{S}$ provides a scale-invariant measure normalized within $[0, 1]$, ensuring a robust and consistent decision boundary across heterogeneous network modules with varying ranks. Finally, $\mathcal{A}_\mathcal{S}$ preserves basis invariance and symmetry. For any orthogonal rotation $Q \in \mathcal{O}(r)$, the relation $\|\tilde{S}_1^\top (\tilde{S}_2 Q)\|_F^2 = \|\tilde{S}_1^\top \tilde{S}_2\|_F^2$ holds, confirming that the measure captures the intrinsic geometry rather than basis artifacts. Concurrently, the cyclic property of the trace ensures geometric symmetry, preventing evaluation bias from the sequential task arrival order.

\subsubsection{Differentiated Expert Evolution}
To accommodate expert correlations as tasks arrive, we employ a distribution-aware adaptation mechanism instead of a fixed threshold. At each step $t$, an adaptive threshold $\gamma_{\mathcal{A}}^{(t)} = \mu^{(t)} + \beta\sigma^{(t)}$ is calibrated for merging decisions based on the historical statistics of affinity scores, where $\mu^{(t)}$ and $\sigma^{(t)}$ denote the mean and standard deviation of the collected scores, and $\beta$ serves as a margin coefficient. Guided by the adaptive thresholds, the modular expert population evolves via two criteria: 

\noindent\textbf{Consolidation.} If $\max_{1\leq k\leq|\mathbf{E}_c|}\mathcal{A}(E_c^{(k)},E_c^{\prime(t)}) \geq \gamma_{\mathcal{A}}^{(t)}$, indicating high functional overlap between the candidate and the most similar expert $E_c^{(k^*)}$, where $k^* = \operatorname{arg\,max}_k \mathcal{A}(E_{c}^{(k)}, E_c^{\prime(t)})$, we mitigate expert redundancy via subspace merging \citep{gargiulo2025task}. We concatenate the bases as ${U}_c = [\tilde{U}_c^{\prime(t)}, \tilde{U}_c^{(k^*)}]$ and ${V}_c = [\tilde{V}_c^{\prime(t)}, \tilde{V}_c^{(k^*)}]$, and form the block-diagonal matrix ${\Sigma}_c = \operatorname{diag}(\tilde{\Sigma}_c^{\prime(t)}, \tilde{\Sigma}_c^{(k^*)})$. After resolving non-orthogonality via polar decomposition to obtain $\hat{U}_c$ and $\hat{V}_c$, the updated expert $E_c^{(k^*)}$ is finally extracted from the top-$r$ singular components of the reconstructed form $\hat{U}_c {\Sigma}_c \hat{V}_c^{\top}$.

\noindent\textbf{Creation.} Otherwise, it implies that the new candidate expert occupies a specific subspace that is not effectively covered by existing experts. Thus, we integrate the new expert into the architecture to preserve diverse expertise and specialization. 

\subsection{Implicit Routing for Expert Activation}
\label{sec:Implicit Routing}
This mechanism is based on our key observations of the geometric alignment between input features and expert principal subspaces, and the hierarchical dependencies among evolved modular experts.

\subsubsection{Feature Projection Alignment}
During inference, to efficiently activate the optimal expert from the current set $\mathbf{E}_c$ at each module, we extract the intermediate activations $\mathbf{h}_c$ from the pre-trained backbone through the forward pass, which serve as a shared input to evaluate the responsiveness of all experts.

We quantify the compatibility between $\mathbf{h}_c$ and an expert $E_{c}^{(k)}$ by focusing on its input-sensitive directions and measuring the match through feature projection alignment (FPA). We first project $\mathbf{h}_c$ onto the principal input subspace of $E_{c}^{(k)}$ using the orthogonal projection operator $\mathbf{P}_{\tilde{V}_c^{(k)}} = \tilde{V}_c^{(k)}\tilde{V}_c^{(k)\top}$, yielding the projected component ${\hat{\mathbf{h}}}_c^{(k)} = \mathbf{P}_{\tilde{V}_c^{(k)}}\mathbf{h}_c$. The FPA score $\mathcal{M}$ is then defined as the cosine similarity between the original input and its projection:
\begin{equation}
    \mathcal{M}(\mathbf{h}_c, E_{c}^{(k)}) = \frac{\mathbf{h}_c^\top {\hat{\mathbf{h}}}_c^{(k)}}{\|\mathbf{h}_c\|_2 \cdot \|{\hat{\mathbf{h}}}_c^{(k)}\|_2}.
    \label{eq:fpa}
\end{equation}
A larger $\mathcal{M} \in [0,1]$ prioritizes experts whose operational domain aligns most closely with the input signal, ensuring that selected experts induce meaningful feature updates.

\subsubsection{Expert Dependency and Path Consistency}
Sole reliance on local expert selection may lead to semantically inconsistent paths across modules. To enforce global coherence, we construct an expert dependency graph $G=(\mathcal{V},\mathcal{L})$, where $\mathcal{V}$ represents the set of all modular experts, and $\mathcal{L}$ denotes valid expert dependency links between adjacent modules. Let $\Lambda(E_{c}^{(k)})$ denote the set of source task identities associated with expert $E_{c}^{(k)}$. A vertical dependency edge $(E_{c}^{(k)}, E_{c+1}^{(j)}) \in \mathcal{L}$ connects experts in adjacent modules if and only if they share a common task origin, i.e., $\Lambda(E_{c}^{(k)}) \cap \Lambda(E_{c+1}^{(j)}) \neq \emptyset$. 

Based on ${G}$, we design an implicit routing strategy to navigate a semantically valid path in a data-free and training-free manner. The process begins by designating the module with the highest expert diversity as the primary anchor $c^{\ast}$, where the optimal expert $E_{c^\ast}^\ast$ with the maximum FPA score (Eq.~\ref{eq:fpa}) is activated. This decision initializes the active source set as a global constraint $\Lambda_{\mathrm{active}} \leftarrow \Lambda(E_{c^\ast}^\ast)$, which is then propagated to other modules along the graph edges. During this traversal, we prune inconsistent branches by retaining only those candidate experts that maintain vertical consistency with $\Lambda_{\mathrm{active}}$, meaning that their source set intersection is non-empty. In cases of pathway bifurcation, where multiple experts in a module $c$ satisfy the propagated constraint, the module is treated as a new anchor. We resolve the ambiguity by selecting the expert with the highest FPA score and refining the constraint $\Lambda_{\mathrm{active}} \leftarrow \Lambda_{\mathrm{active}} \cap \Lambda(E_{c}^\ast)$. This cycle of constraint propagation and conflict resolution iterates recursively until a unique and globally consistent path is established.

\section{Experiments}
\subsection{Experiment Setup}
\label{exp:setup}
\noindent\textbf{Models and Datasets.}
Following \citep{tang2025merging}, we employ CLIP-ViT \citep{radford2021learning} backbones and construct 8, 14, and 20-task groups for merging ViT-B/32, ViT-B/16, and ViT-L/14 models to evaluate scalability and robustness. For fair comparison, we adopt publicly available fine-tuned checkpoints \citep{tang2024fusionbench}. Detailed settings are provided in Appendix \ref{ap:setup}.

\noindent\textbf{Evaluation Metrics.}
We adopt average accuracy (ACC) and backward transfer (BWT) \citep{lin2022beyond} for evaluation. ACC is the average accuracy across all tasks: $\mathrm{ACC}=\frac{1}{T}\sum_{i=1}^TA_i(\theta_\mathrm{merged}^{(T)})$. BWT quantifies forgetting through the average performance degradation after the final merge: $\mathrm{BWT}=\frac{1}{T-1}\sum_{i=1}^{T-1}[A_i(\theta_{\mathrm{merged}}^{(T)})-A_i(\theta_{\mathrm{merged}}^{(i)})]$, where $A_i(\cdot)$ denotes the accuracy on the $i$-th task. 

\noindent\textbf{Baselines.}
We compare MADE-IT with three categories of baselines: (1) Non-merging paradigms including pre-training, standard fine-tuning, and continual fine-tuning; (2) Continual adaptations of conventional merging methods including Stochastic Weight Averaging (SWA) \citep{izmailov2018averaging}, Task Arithmetic \citep{ilharco2023editing}, Ties-Merging \citep{yadav2023ties}, Layer-Wise AdaMerging \citep{yangadamerging}, and LoRA-WEMoE \citep{shen2026efficient}; and (3) Continual merging methods including OPCM \citep{tang2025merging} and MINGLE \citep{qiu2025mingle}. 

\noindent\textbf{Implementation Details.}
To evaluate task order sensitivity, we conduct 10 runs for each experiment using random seeds ranging from 42 to 51. All experiments share global hyper-parameters across model architectures and task sequences, with rank ratio $\rho=0.1$ and margin coefficient $\beta=1.0$. Sensitivity analysis is provided in Section \ref{sec:Sensitivity Analysis}.

\begin{table*}[htbp]
\footnotesize
\setlength{\tabcolsep}{1.5pt}
\caption{Comparative results of MADE-IT against three categories of baselines across three CLIP-ViT architectures. We report the average accuracy (ACC) and backward transfer (BWT) averaged over ten task orders (mean$\pm$std). Best results are in bold. 'Continual' is abbreviated as 'C.' for conciseness.}
\label{tab:main_results}
\begin{tabular}{ccccccccccc}
\hline
\multicolumn{2}{c}{\multirow{2}{*}{Method}}      & \multicolumn{3}{c}{ViT-B/32}                              & \multicolumn{3}{c}{ViT-B/16}                              & \multicolumn{3}{c}{ViT-L/14}                              \\ \cline{3-11} 
\multicolumn{2}{c}{}                             & 8 tasks           & 14 tasks          & 20 tasks          & 8 tasks           & 14 tasks          & 20 tasks          & 8 tasks           & 14 tasks          & 20 tasks          \\ \hline
                            & Pre-Trained        & 48.1              & 56.9              & 55.6              & 55.4              & 62.0              & 59.8              & 64.9              & 69.1              & 65.6              \\
                            & Fine-Tuned         & 90.4              & 89.3              & 89.8              & 92.4              & 91.3              & 91.6              & 94.3              & 93.4              & 93.5              \\
                            & C. Fine-Tuned      & 79.8              & 67.4              & 62.6              & 82.9              & 72.2              & 68.2              & 90.0                & 70.9              & 77.7              \\ \hline
\multirow{8}{*}{\rotatebox{90}{ACC (\%) ↑}} & Average (SWA)      & 66.3±0.0          & 65.4±0.0          & 61.1±0.0          & 72.3±0.0          & 69.7±0.0          & 64.8±0.0          & 80.0±0.0          & 77.5±0.0          & 71.1±0.0          \\
                            & C. Task Arithmetic & 67.5±0.0          & 66.5±0.0          & 60.0±0.0          & 77.1±0.0          & 70.9±0.6          & 64.2±0.0          & 82.1±0.0          & 77.9±0.0          & 70.3±0.0          \\
                            & C. Ties-Merging    & 49.0±10.2         & 66.2±0.6          & 59.9±0.7          & 66.8±3.7          & 70.5±0.8          & 63.0±1.6          & 64.3±7.0          & 78.0±0.6          & 68.3±0.9          \\
                            & C. LW AdaMerging   & 53.4±3.2          & 59.8±1.6          & 59.7±7.4          & 59.9±2.3          & 64.3±1.2          & 61.5±1.1          & 68.8±2.9          & 73.1±5.7          & 66.9±1.1          \\
                            & C. LoRA-WEMoE      & 68.8±7.8          & 63.8±3.4          & 49.6±15.4         & 72.6±3.7          & 67.9±2.9          & 55.0±7.0          & 75.6±7.8          & 74.0±5.0          & 56.9±19.8         \\
                            & OPCM               & 75.5±0.5          & 71.9±0.3          & 65.7±0.2          & 81.8±0.3          & 77.1±0.5          & 70.3±0.2          & 87.0±0.4          & 83.5±0.2          & 76.0±0.2          \\
                            & MINGLE             & 85.8±0.8          & 81.6±1.4          & 77.1±2.0          & 88.3±0.6          & 84.9±0.8          & 81.9±0.9          & 91.8±0.2          & 88.8±0.7          & 85.5±1.3          \\
                            & \textbf{MADE-IT (Ours)}     & \textbf{87.1}±1.6 & \textbf{84.3}±0.6 & \textbf{81.4}±1.3 & \textbf{90.4}±0.6 & \textbf{87.7}±0.4 & \textbf{83.0}±0.1 & \textbf{93.1}±0.7 & \textbf{91.7}±0.7 & \textbf{89.6}±0.7 \\ \hline
\multirow{8}{*}{\rotatebox{90}{BWT (\%) ↑}} & Average (SWA)      & -11.5±2.2         & -8.0±1.3          & -7.1±2.1          & -9.7±1.5          & -7.1±1.4          & -7.3±1.7          & -7.3±1.4          & -5.8±1.0          & -6.4±1.5          \\
                            & C. Task Arithmetic & -9.6±1.5          & -1.3±1.6          & -3.4±1.0          & -4.2±1.0          & -1.3±0.4          & -3.6±0.4          & -7.1±0.8          & -1.8±0.3          & -3.3±0.3          \\
                            & C. Ties-Merging    & -15.3±8.0         & \textbf{1.9}±0.6  & \textbf{-1.5}±0.7 & -5.5±0.4          & \textbf{1.4}±0.7  & \textbf{-1.5}±1.2 & -13.0±5.7         & -1.1±0.4          & -2.9±1.0          \\
                            & C. LW AdaMerging   & -32.5±3.6         & -24.1±1.7         & -22.7±4.3         & -27.8±2.7         & -22.1±1.4         & -21.4±1.2         & -24.3±3.3         & -19.6±1.7         & -21.7±1.1         \\
                            & C. LoRA-WEMoE      & -20.4±9.0         & -20.2±3.9         & -24.5±10.0        & -18.0±6.2         & -18.8±3.4         & -25.8±7.9         & -17.8±5.9         & -16.8±5.3         & -27.9±17.2        \\
                            & OPCM               & -6.3±1.1          & -6.0±1.0          & -7.8±1.5          & -4.8±0.7          & -5.1±1.4          & -6.3±2.2          & -2.6±1.0          & -4.3±0.7          & -6.5±1.8          \\
                            & MINGLE             & \textbf{-0.6}±0.4 & -1.1±0.3          & -2.2±0.8          & \textbf{-0.4}±0.1 & -0.9±0.1          & -1.9±0.4          & \textbf{-0.6}±0.1 & \textbf{-1.0}±0.3 & -2.6±0.9          \\
                            & \textbf{MADE-IT (Ours)}     & -2.2±1.2          & -2.6±0.6          & -3.8±3.1          & -1.2±0.6          & -1.9±0.6          & -3.0±0.9          & \textbf{-0.6}±0.4 & \textbf{-1.0}±0.5 & \textbf{-1.7}±0.9 \\ \hline
\end{tabular}
\label{table:main}
\end{table*}

\subsection{Main Results}
\label{exp:main}

As detailed in Table~\ref{table:main}, MADE-IT consistently outperforms merging baselines across all architectures and sequence lengths in ACC, while preserving competitive BWT scores. In the challenging 20-task scenario, MADE-IT achieves 81.4\% accuracy with the ViT-B/32 backbone, outperforming the strongest baseline MINGLE by a significant margin of 4.3\% and surpassing OPCM by a substantial 15.7\%. On the ViT-B/16 and the larger ViT-L/14 architectures, our method establishes new state-of-the-art accuracies, further widening the performance gap over existing techniques. Notably, MADE-IT significantly narrows the disparity between merged models and the theoretical upper bound of individually fine-tuned models. These results underscore the effectiveness of our method in achieving higher accuracy and long-horizon stability without data access or auxiliary training.

\subsection{Analysis of Manifold-Aware Expert Evolution}
\label{exp:analysis}

\noindent\textbf{Generic-to-Specific Expert Allocation.}
Fig.~\ref{fig:expert_chain} visualizes the module-wise expert allocation trajectories across 20 tasks for the ViT-L/14 architecture, revealing a distinct generic-to-specific hierarchical evolution pattern consistently observed across all architectures (see Appendix \ref{ap:results} for comprehensive visualizations and analysis). Shallower modules predominantly share a compact set of universal experts by merging redundant ones with high geometric affinity. Conversely, deeper layers transition towards expert specialization to preserve task-specific fidelity. This adaptive allocation autonomously resolves the saturation-redundancy dilemma by harmonizing shallow-layer parameter efficiency with deep-layer plasticity, closely aligning with the feature hierarchy of deep networks. 

\begin{figure*}[htbp]
    \centering
    \includegraphics[width=1.0\linewidth]{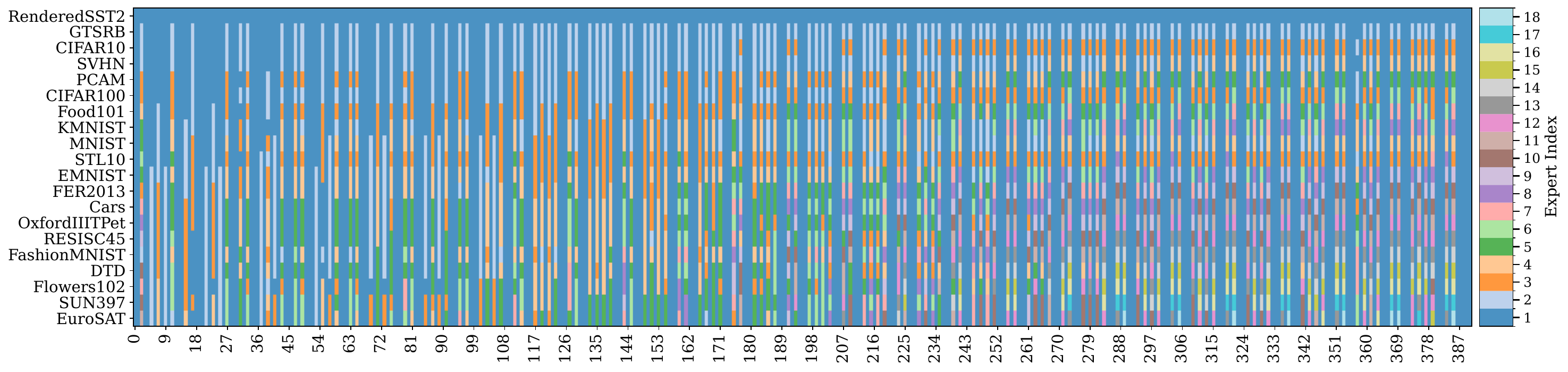}
    \caption{Module-wise visualization of dynamic expert evolution across 20 sequential tasks using ViT-L/14. In each column, consistent colors indicate shared experts among tasks within the module.}
    \label{fig:expert_chain}
\end{figure*}

\noindent\textbf{Quantitative Analysis of Expert Evolution.}
Fig.~\ref{fig:expert_analysis} provides a quantitative allocation breakdown for the ViT-B/32 architecture to assess parameter efficiency. Component-wise (Fig.~\ref{expert_analysis:sub1}), MLP modules exhibit the highest expert retention, highlighting their pivotal role in encoding task-specific knowledge, whereas generic modules achieve reduction rates near 95\%. Depth-wise (Fig.~\ref{expert_analysis:sub2}), expert retention increases toward deeper layers. This adaptive allocation confirms that MADE-IT maximizes parameter efficiency through expert sharing in generic modules and early layers, while preserving multi-task capacity via increased diversity in task-sensitive modules and deeper layers.

\begin{figure*}[htbp]
\centering
\subfloat[]{%
    \includegraphics[width=0.4\linewidth]{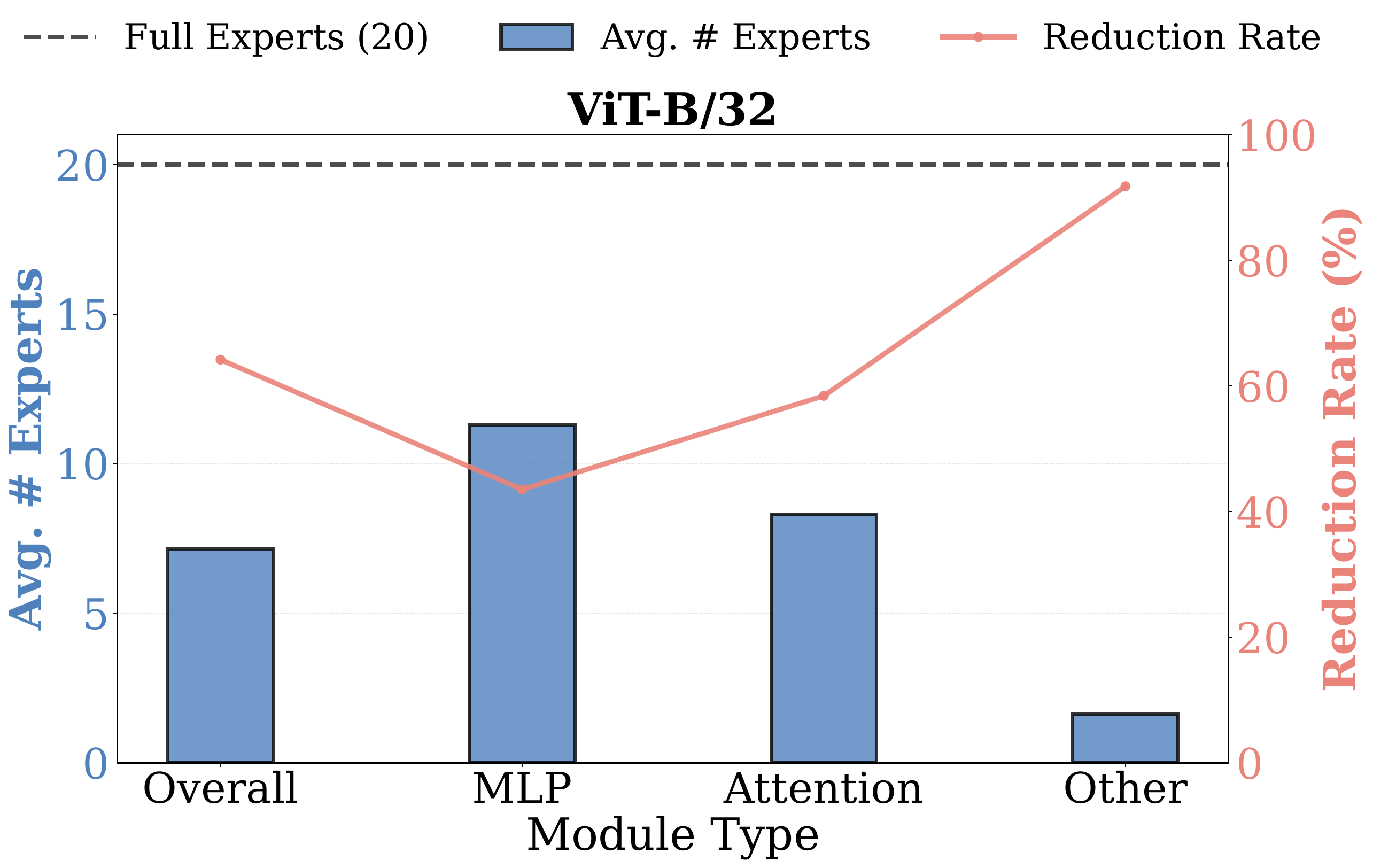}%
    \label{expert_analysis:sub1}}
\hspace{0.08\linewidth}
\subfloat[]{%
    \includegraphics[width=0.4\linewidth]{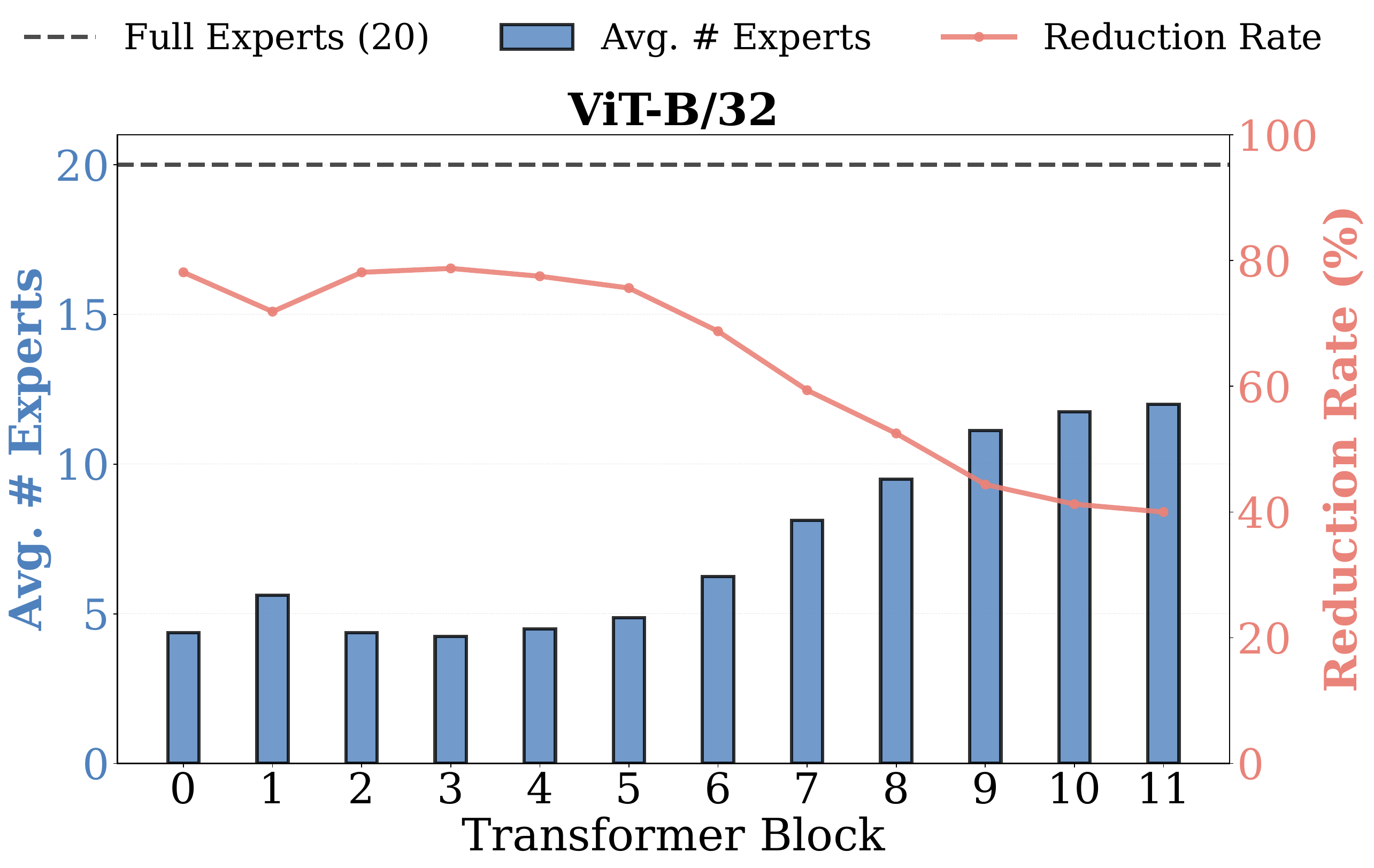}%
    \label{expert_analysis:sub2}}
\caption{Quantitative analysis of dynamic expert evolution across components and network depths.} 
\label{fig:expert_analysis}
\end{figure*}

\subsection{Analysis of Projection-Based Subspace Affinity}
\noindent\textbf{Quantitative Performance Evaluation.}
Table~\ref{similarity_metrics} demonstrates that MADE-IT using subspace affinity consistently outperforms its cosine similarity variant across all architectures, yielding significantly higher accuracy and reduced catastrophic forgetting. This robust superiority confirms that, unlike parameter-space metrics susceptible to ambient high-dimensional noise, our geometric measure effectively captures intrinsic inter-expert correlations to orchestrate optimal expert evolution.

\begin{table*}[htbp]
\small
\setlength{\tabcolsep}{1.8pt}
\caption{Quantitative comparison of similarity metrics across all architectures in the 8-task setting.}
\begin{tabular}{ccccccc}
\hline
\multirow{2}{*}{Strategy}                           & \multicolumn{2}{c}{ViT-B/32}          & \multicolumn{2}{c}{ViT-B/16}          & \multicolumn{2}{c}{ViT-L/14}          \\ \cline{2-7} 
                                                   & ACC               & BWT               & ACC               & BWT               & ACC               & BWT               \\ \hline
w/ Cosine Similarity                                  & 84.0±2.6          & -4.5±2.4          & 85.1±2.2          & -4.9±2.2          & 92.1±0.7          & -1.5±0.6          \\
\textbf{w/ Projection-Based Subspace Affinity (Ours)} & \textbf{87.1±1.6} & \textbf{-2.2±1.2} & \textbf{90.4±0.6} & \textbf{-1.2±0.6} & \textbf{93.1±0.7} & \textbf{-0.6±0.4} \\ \hline
\end{tabular}
\label{similarity_metrics}
\end{table*}

\begin{wrapfigure}{r}{0.4\textwidth} %
    \centering
    \includegraphics[width=\linewidth]{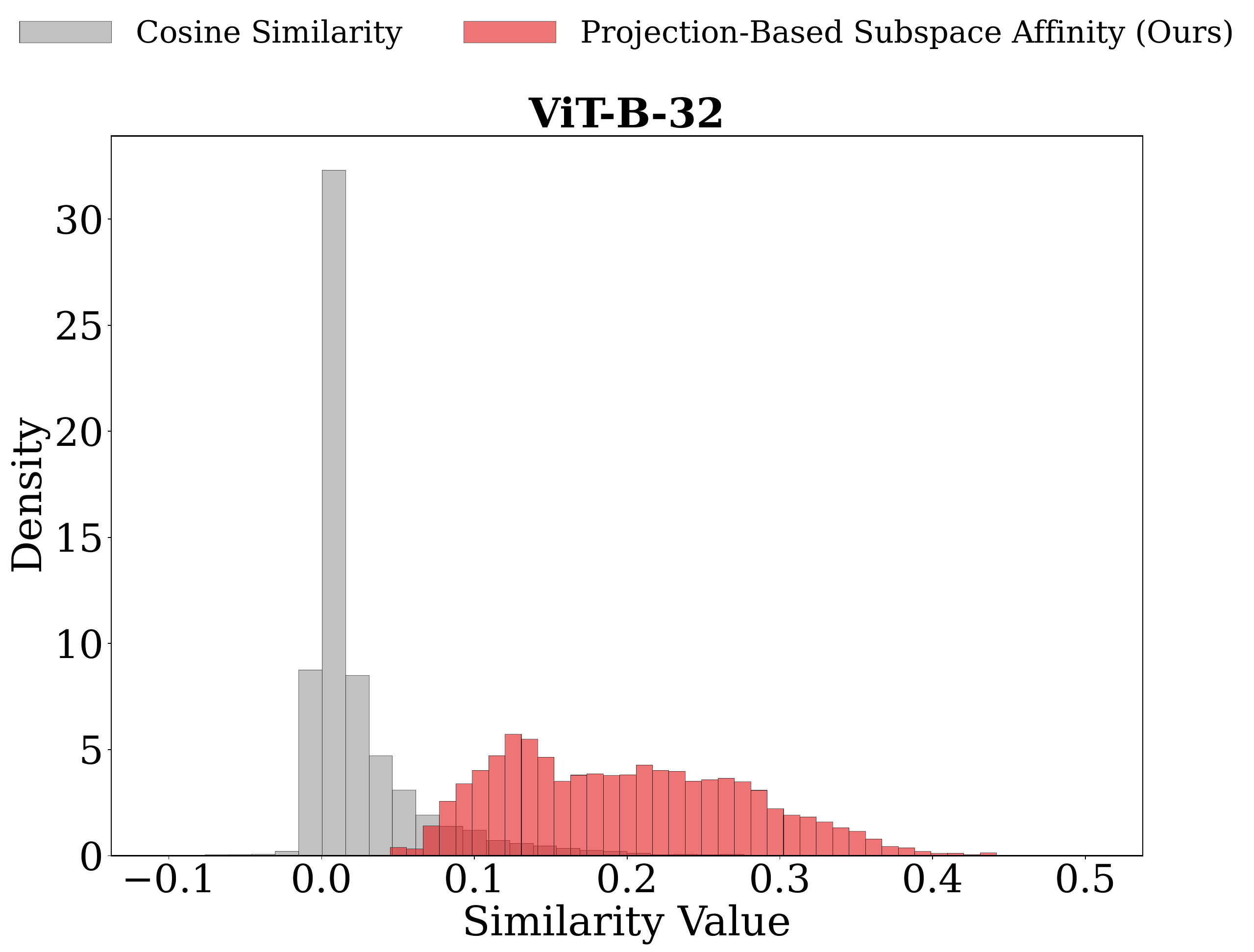}
    \caption{Density distribution of pairwise similarity scores across 20 models.}
    \label{fig:Metric_Distribution}
    \vspace{-5mm} 
\end{wrapfigure}

\noindent\textbf{Distributional Analysis.}
Fig.~\ref{fig:Metric_Distribution} compares the density distribution of pairwise similarity scores. Comparative results reveal that cosine similarity scores cluster tightly around zero due to the inherent orthogonality of high-dimensional parameters, rendering it incapable of distinguishing subtle inter-expert variations. Conversely, our projection-based subspace affinity exhibits a remarkably smoother distribution with a broader dynamic range. This enhanced discriminability unlocks rich informational signals of inter-expert correlations for dynamic evolution. 

\noindent\textbf{Visualization of Semantic Correlations.}
Fig.~\ref{fig:heatmap} provides microscopic validation via task-wise similarity heatmaps. While cosine similarity yields near-zero off-diagonal values that obscure latent relationships, our subspace affinity naturally manifests distinct clusters (e.g., the MNIST family). This confirms that our metric guides the collaboration of functionally aligned experts with high geometric correlations.

\begin{figure*}[htbp]
\centering
\subfloat[]{%
    \includegraphics[width=0.49\linewidth]{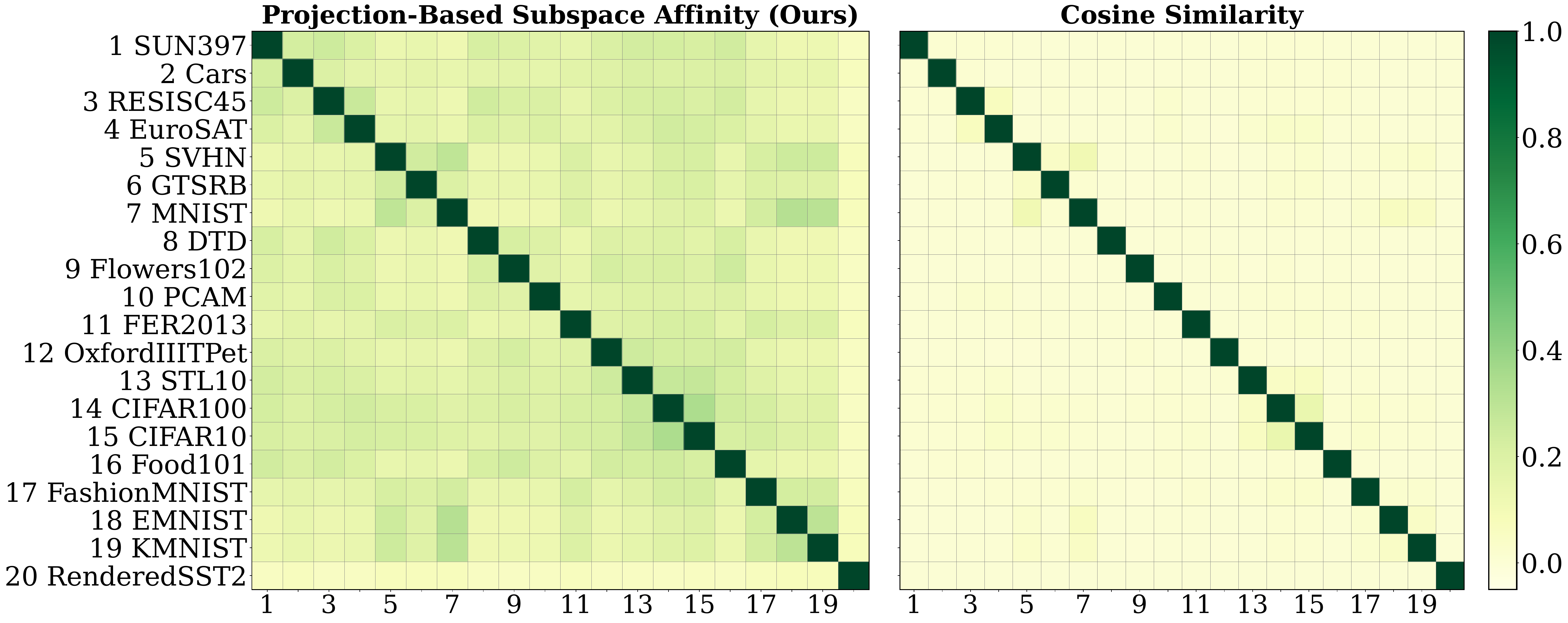}%
    \label{heatmap:sub1}}
\hfil
\subfloat[]{%
    \includegraphics[width=0.49\linewidth]{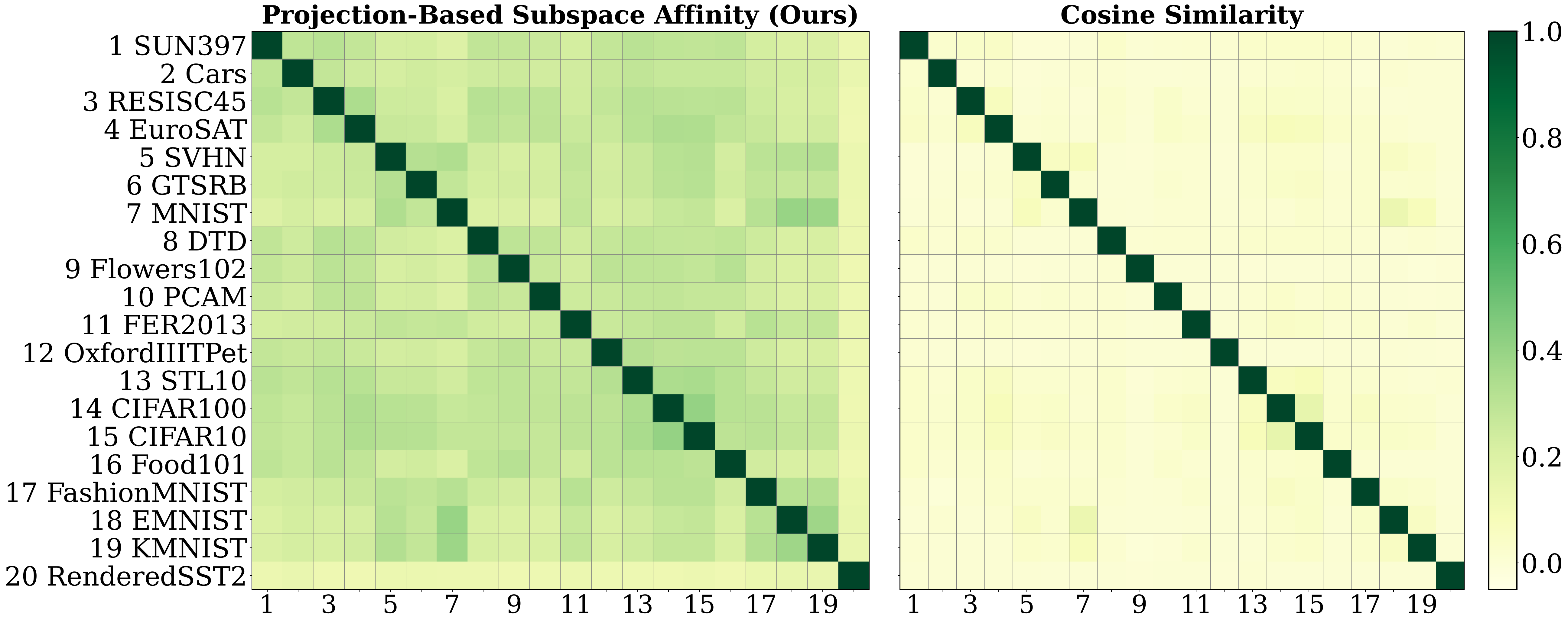}%
    \label{heatmap:sub2}}
\caption{Similarity heatmaps for (a) mlp.fc1 and (b) self\_attn.k.proj in the 6-th layer of ViT-B/32.}
\label{fig:heatmap}
\end{figure*}

\subsection{Hyperparameter Sensitivity Analysis}
\label{sec:Sensitivity Analysis}
\noindent\textbf{Sensitivity to Rank Ratio $\rho$.}
As shown in Fig.~\ref{hyperR}, performance plateaus within $\rho \leq 0.6$ and precipitously declines as $\rho$ approaches unity. Crucially, unconstrained ranks inject noise as expert subspaces degenerate toward the full parameter space, deteriorating the feature discriminability essential for implicit routing. Consequently, adopting a compact $\rho=0.1$ enhances both parameter efficiency and expert specificity via rigorous denoising, ensuring multi-task performance.

\begin{figure*}[htbp]
\centering
\subfloat[]{%
    \includegraphics[width=0.4\linewidth]{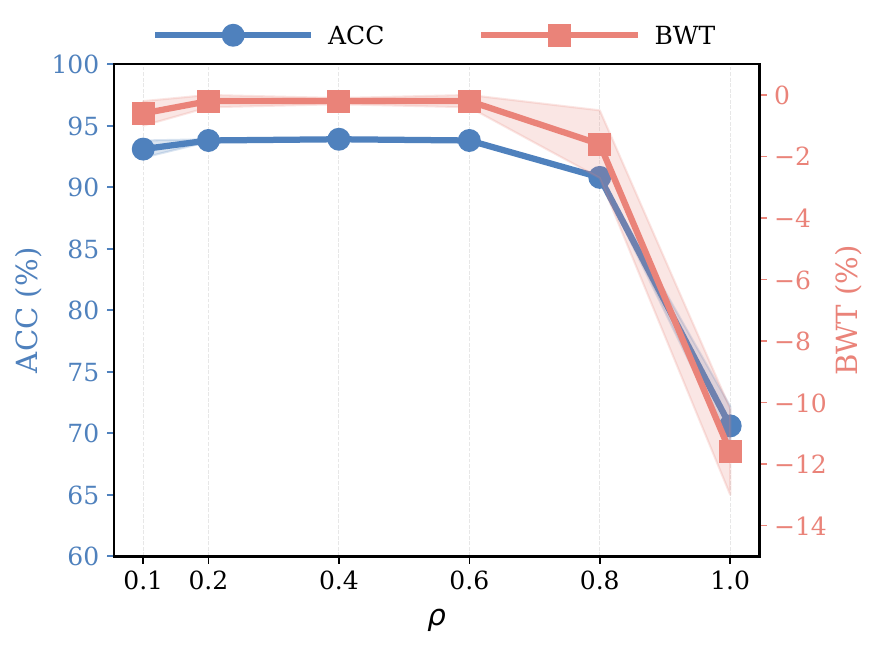}%
    \label{hyperR}}
\hspace{0.08\linewidth}
\subfloat[]{%
    \includegraphics[width=0.4\linewidth]{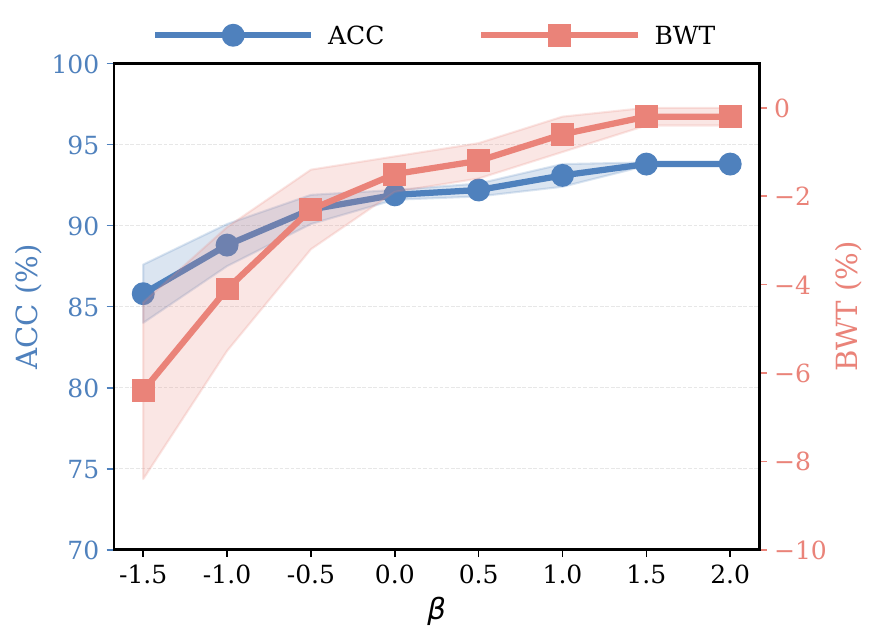}%
    \label{hyperK}}
\caption{Sensitivity analysis of the (a) rank ratio $\rho$ and (b) margin coefficient $\beta$ across the 8-task setting on the ViT-L/14 architecture. Shaded areas denote the standard deviation.}
\label{fig:hyperR&k}
\end{figure*}

\noindent\textbf{Sensitivity to Margin Coefficient $\beta$.}
As shown in Fig.~\ref{hyperK}, increasing $\beta$ relaxes the creation criterion, facilitating the formation of diverse expert ensembles that enhance cross-task performance. However, a performance plateau emerges beyond $\beta=1.0$, where further threshold elevation yields diminishing returns and inflates the parameter footprint due to redundant experts. Consequently, adopting $\beta=1.0$ balances representational fidelity with architectural parsimony across complex task sequences.

\section{Conclusion}

This paper resolves the saturation-redundancy dilemma and the routing bottleneck in continual merging by proposing MADE-IT, which reframes expert management and activation through manifold geometry. MADE-IT utilizes robust subspace affinity to guide autonomous expert evolution, harmonizing expert diversity with architectural parsimony. Furthermore, its inherently training-free implicit routing mechanism achieves efficient expert activation via feature-subspace alignment. Extensive evaluations confirm that MADE-IT significantly enhances multi-task accuracy while mitigating catastrophic forgetting, establishing it as an adaptive and scalable continual merging approach.
\bibliographystyle{ACM-Reference-Format}
\bibliography{cmm}


\appendix

\section{Theoretical Properties of Subspace Affinity}
\label{ap:theory}
This section provides detailed derivations for the theoretical properties of our proposed subspace affinity. Let $\tilde{S}_1, \tilde{S}_2 \in \mathbb{R}^{d \times r}$ denote the orthonormal basis matrices spanning two $r$-dimensional subspaces $\mathcal{S}_1$ and $\mathcal{S}_2$. The orthogonal projection operators are defined as $\mathbf{P}_1 = \tilde{S}_1\tilde{S}_1^\top$ and $\mathbf{P}_2 = \tilde{S}_2\tilde{S}_2^\top$. The individual affinity component is defined as $\mathcal{A}_\mathcal{S} = \frac{1}{r} \| \tilde{S}_1^\top \tilde{S}_2 \|_F^2$.

\subsection{Trace-Frobenius Equivalence}
The computational efficiency of $\mathcal{A}_\mathcal{S}$ is facilitated by the equivalence between the Hilbert-Schmidt inner product of projection operators and the squared Frobenius norm of their basis interaction matrix.

\begin{lemma}[Trace-Frobenius Equivalence]
Given projection operators $\mathbf{P}_1, \mathbf{P}_2$ and their orthonormal bases $\tilde{S}_1, \tilde{S}_2$, the following identity holds:
\begin{equation}
    \operatorname{Tr}(\mathbf{P}_1 \mathbf{P}_2) = \| \tilde{S}_1^\top \tilde{S}_2 \|_F^2.
\end{equation}
\end{lemma}

\begin{proof}
Expanding the trace of the product of the projection operators yields:
\begin{equation}
    \operatorname{Tr}(\mathbf{P}_1 \mathbf{P}_2) = \operatorname{Tr}(\tilde{S}_1 \tilde{S}_1^\top \tilde{S}_2 \tilde{S}_2^\top).
\end{equation}
Leveraging the cyclic property of the trace, $\operatorname{Tr}(ABCD) = \operatorname{Tr}(DABC)$, we cyclically permute the matrices:
\begin{equation}
    \operatorname{Tr}(\tilde{S}_1 \tilde{S}_1^\top \tilde{S}_2 \tilde{S}_2^\top) = \operatorname{Tr}(\tilde{S}_2^\top \tilde{S}_1 \tilde{S}_1^\top \tilde{S}_2)= \operatorname{Tr}((\tilde{S}_1^\top \tilde{S}_2)^\top (\tilde{S}_1^\top \tilde{S}_2)).
\end{equation}
Recall that for any matrix $M$, $\operatorname{Tr}(M^\top M) = \|M\|_F^2$. Setting $M = \tilde{S}_1^\top \tilde{S}_2 \in \mathbb{R}^{r \times r}$, we obtain:
\begin{equation}
    \operatorname{Tr}(\mathbf{P}_1 \mathbf{P}_2) = \| \tilde{S}_1^\top \tilde{S}_2 \|_F^2,
\end{equation}
which proves the lemma. This identity justifies the computational reduction from the $d \times d$ ambient parameter space to an $r \times r$ compact subspace interaction.
\end{proof}

\subsection{Relation to Principal Angles}

\begin{proposition}[Spectral Interpretation]
Let $0 \le \phi_1 \le \dots \le \phi_r \le \pi/2$ be the principal angles between $\mathcal{S}_1$ and $\mathcal{S}_2$. The affinity $\mathcal{A}_\mathcal{S}$ represents the mean squared cosine of these principal angles:
\begin{equation}
    \mathcal{A}_\mathcal{S} = \frac{1}{r} \sum_{i=1}^r \cos^2 \phi_i.
\end{equation}
\end{proposition}

\begin{proof}
In manifold geometry, the cosines of the principal angles $\{\cos \phi_i\}_{i=1}^r$ are equal to the singular values $\{\sigma_i\}_{i=1}^r$ of the basis interaction matrix $M = \tilde{S}_1^\top \tilde{S}_2$. Let the singular value decomposition of $M$ be $M = U \Sigma V^\top$, where $U, V \in \mathcal{O}(r)$ are orthogonal matrices, and $\Sigma = \operatorname{diag}(\cos \phi_1, \dots, \cos \phi_r)$. The squared Frobenius norm of $M$ can be expanded as: 
\begin{equation}
    \| M \|_F^2 = \operatorname{Tr}(M^\top M) = \operatorname{Tr}(V \Sigma U^\top U \Sigma V^\top).
\end{equation}
Applying $U^\top U = I_r$, $V^\top V = I_r$ and leveraging the cyclic property of the trace, this simplifies to:
\begin{equation}
    \| M \|_F^2 = \operatorname{Tr}(\Sigma^2) = \sum_{i=1}^r \cos^2 \phi_i.
\end{equation}
Dividing by the rank $r$ yields $\mathcal{A}_\mathcal{S} = \frac{1}{r} \sum_{i=1}^r \cos^2 \phi_i$, completing the proof. This confirms that $\mathcal{A}_\mathcal{S} \in [0, 1]$ serves as a normalized measure of total spectral alignment.
\end{proof}

\subsection{Relation to Chordal Distance}

\begin{proposition}[Equivalence to Normalized Chordal Distance]
\label{prop:chordal_equiv}
The affinity $\mathcal{A}_\mathcal{S}$ is linearly related to the squared Chordal distance $d_C^2 = \frac{1}{2} \| \mathbf{P}_1 - \mathbf{P}_2 \|_F^2$ via the mapping:
\begin{equation}
    \mathcal{A}_\mathcal{S} = 1 - \frac{1}{r} d_C^2.
\end{equation}
\end{proposition}

\begin{proof}
Expanding the Frobenius norm definition of the Chordal distance yields:
\begin{align}
    d_C^2 &= \frac{1}{2} \operatorname{Tr}( (\mathbf{P}_1 - \mathbf{P}_2)^\top (\mathbf{P}_1 - \mathbf{P}_2) ) \nonumber \\
    &= \frac{1}{2} ( \operatorname{Tr}(\mathbf{P}_1^2) + \operatorname{Tr}(\mathbf{P}_2^2) - 2\operatorname{Tr}(\mathbf{P}_1 \mathbf{P}_2) ).
\end{align}
Utilizing the idempotence ($\mathbf{P}^2 = \mathbf{P}$) and the property that $\operatorname{Tr}(\mathbf{P}) = \operatorname{rank}(\mathbf{P}) = r$:
\begin{equation}
    d_C^2 = \frac{1}{2} (r + r - 2\operatorname{Tr}(\mathbf{P}_1 \mathbf{P}_2)) = r - \operatorname{Tr}(\mathbf{P}_1 \mathbf{P}_2).
\end{equation}
By applying the Trace-Frobenius Lemma, we substitute $\operatorname{Tr}(\mathbf{P}_1 \mathbf{P}_2) = \| \tilde{S}_1^\top \tilde{S}_2 \|_F^2$:
\begin{equation}
    d_C^2 = r - \| \tilde{S}_1^\top \tilde{S}_2 \|_F^2.
\end{equation}
Substituting the definition $\mathcal{A}_\mathcal{S} = \frac{1}{r} \| \tilde{S}_1^\top \tilde{S}_2 \|_F^2$, we obtain $\mathcal{A}_\mathcal{S} = 1 - \frac{1}{r} d_C^2$, completing the proof. This confirms that $\mathcal{A}_\mathcal{S}$ operates as a scale-invariant metric for subspace proximity.
\end{proof}

\section{Workflow of MADE-IT}
\label{ap:algorithm}
We summarize the main workflow of MADE-IT in Algorithm \ref{alg:made_it}, which comprises two primary phases: Manifold-Aware Dynamic Expert Evolution and Implicit Routing. 

\noindent\textbf{Dynamic Expert Evolution.}
The procedure commences by initializing the expert repositories for all target modules (line 1). Upon the arrival of each fine-tuned model for a specific task, the algorithm extracts the principal subspace from the module-wise weight updates via truncated SVD to construct a candidate modular expert (lines 4-5). Subsequently, it computes the projection-based subspace affinity between the newly extracted candidate expert and all existing experts within the respective module (lines 6-7). Based on an adaptive threshold calibrated from historical statistics (lines 8-9), the algorithm performs branching logic for expert evolution: if the maximum affinity is below the threshold, the mechanism instantiates the candidate as a new expert and assigns its task identity (lines 11-13); otherwise, it consolidates the candidate into the most similar reference expert via subspace merging and updates the corresponding task identity set (lines 14-18). Finally, the globally maintained expert set is updated (lines 20-21).

\noindent\textbf{Implicit Routing.}
During multi-task inference, we first initialize the activated expert set (line 23), identify the anchor module with the highest expert diversity (line 24), and perform a backbone forward pass to obtain intermediate features (line 25). The optimal expert for the anchor module is activated based on the Feature Projection Alignment scores, and the global pathway constraint is initialized according to its task identity set (lines 26-29). This constraint is then propagated across the expert dependency graph to activate functionally consistent candidates and resolve potential path bifurcations in other modules (lines 30-37). Finally, the activated experts are integrated into the pre-trained backbone for model inference (line 38).

\begin{algorithm}[htbp]
\caption{MADE-IT Procedure.}
\label{alg:made_it}
\textbf{Input}: Pre-trained model $\theta^{(0)}$, sequentially arriving fine-tuned models $\{\theta^{(t)}\}_{t=1}^T$, target module set $\mathcal{C}$, rank ratio $\rho$, margin coefficient $\beta$, input $\boldsymbol{x}$\\
\textbf{Output}: Expert set $\mathbf{E}^{(T)}$, output $\boldsymbol{y}$
\begin{algorithmic}[1]
\STATE Initialize: $\mathbf{E}^{(0)}= \{\mathbf{E}_c\}_{c \in \mathcal{C}} \leftarrow \emptyset$ \hfill \textbf{\# Expert Evolution}
\FOR {each task $t=1$ to $T$}
    \FOR {each module $c \in \mathcal{C}$}
        \STATE {\scriptsize{$\triangleright$ extract principal subspace} via $\operatorname{SVD}_r(W_c^{(t)} - W_c^{(0)})$ with low rank $r \leftarrow \lfloor \rho \cdot \min(d_i, d_o) \rfloor$}
        \STATE $E_c^{\prime(t)} \leftarrow  \tilde{U}_c^{\prime(t)} \tilde{\Sigma}_c^{\prime(t)} \tilde{V}_c^{\prime(t)\top}$ via Eq.(\ref{equation:E})
        
        \STATE {\scriptsize{$\triangleright$ compute subspace affinity with existing experts}}
        \STATE Compute $\mathcal{A}(E_{c}^{(k)}, E_c^{\prime(t)})$, for $E_{c}^{(k)} \in \mathbf{E}_{c}$ via Eq.(\ref{equation:A})

        \STATE {\scriptsize{$\triangleright$ adapt threshold} with mean $\mu$ and std $\sigma$ of historical affinity scores}
        \STATE $\gamma_{\mathcal{A}}^{(t)} = \mu^{(t)} + \beta\sigma^{(t)}$
        
        \STATE {\scriptsize{$\triangleright$ evolve experts}}
        
        \IF{$\mathbf{E}_c = \emptyset \lor \displaystyle \max_{1\leq k\leq|\mathbf{E}_c|}\mathcal{A}(E_c^{(k)},E_c^{\prime(t)}) < \gamma_{\mathcal{A}}^{(t)}$} 
            \STATE $\mathbf{E}_c \leftarrow \mathbf{E}_c \cup \{E_c^{\prime(t)}\}$
            \STATE $\Lambda(E_c^{\prime(t)}) \leftarrow \{t\}$
        \ELSE
            \STATE $k^* = \operatorname{arg\,max}_k \mathcal{A}(E_{c}^{(k)}, E_c^{\prime(t)})$
            \STATE $E_c^{(k^*)} \leftarrow \text{SubspaceMerge}(E_c^{(k^*)},E_c^{\prime(t)})$
            \STATE $\Lambda(E_c^{(k^*)}) \leftarrow \Lambda(E_c^{(k^*)}) \cup \{t\}$
        \ENDIF
    \ENDFOR
    \STATE {\scriptsize{$\triangleright$ update expert set}}
    \STATE $\mathbf{E}^{(t)} \leftarrow \{\mathbf{E}_c \}_{c \in \mathcal{C}}$
\ENDFOR
\STATE Initialize: $\mathbf{E}_{\text{activated}} \leftarrow \emptyset$ \hfill \textbf{\# Implicit Routing}
\STATE $c^* \leftarrow \arg\max_{c \in \mathcal{C}} |\mathbf{E}_c|$
\STATE $\{\mathbf{h}_c\}_{c \in \mathcal{C}} \leftarrow \operatorname{Forward}(\boldsymbol{x}; \theta^{(0)})$
\STATE {\scriptsize{$\triangleright$ activate expert by computing projection alignment scores}}
\STATE $E_{c^*}^* \leftarrow \arg\max_{E \in \mathbf{E}_{c^*}} \mathcal{M}(\mathbf{h}_{c^*}, E)$ via Eq.~(\ref{eq:fpa})
\STATE $\Lambda_{\text{active}} \leftarrow \Lambda(E_{c^*}^*)$
\STATE $\mathbf{E}_{\text{activated}} \leftarrow \mathbf{E}_{\text{activated}} \cup \{E_{c^*}^*\}$ 
\FOR{$c \in \mathcal{C} \setminus \{c^*\}$}
    \STATE $\mathbf{E}_c^{*} \leftarrow \{E_c^{(k)} \mid \Lambda(E_c^{(k)}) \cap \Lambda_{\text{active}} \neq \emptyset\}$
    \IF{$|\mathbf{E}_c^{*}| > 1$}
        \STATE $\mathbf{E}_{c}^* \leftarrow \{\arg\max_{E \in \mathbf{E}_{c}^{*}} \mathcal{M}(h_{c}, E)\}$
    \ENDIF
    \STATE $\Lambda_{\text{active}} \leftarrow \Lambda_{\text{active}} \cap \Lambda(\mathbf{E}_{c}^*)$
    \STATE $\mathbf{E}_{\text{activated}} \leftarrow \mathbf{E}_{\text{activated}} \cup \mathbf{E}_{c}^*$ 
\ENDFOR
\STATE $\boldsymbol{y} \leftarrow f(\boldsymbol{x}; \operatorname{Integrate}(\theta^{(0)}, \mathbf{E}_{\text{activated}}))$
\end{algorithmic}
\end{algorithm}

\section{Details of Experiment Settings}
\label{ap:setup}
\subsection{Model and Dataset Details}
Following the experimental setup from \citep{tang2025merging}, we use ViT-B/32, ViT-B/16 and ViT-L/14 models, each independently fine-tuned on 20 image classification datasets. These vision tasks include SUN397 \citep{xiao2010sun}, Stanford Cars \citep{krause20133d}, RESISC45 \citep{cheng2017remote}, EuroSAT \citep{helber2019eurosat}, SVHN \citep{netzer2011reading}, GTSRB \citep{stallkamp2012man}, MNIST \citep{lecun2002gradient}, DTD \citep{cimpoi2014describing}, Flowers102 \citep{nilsback2008automated}, PCAM \citep{veeling2018rotation}, FER2013 \citep{goodfellow2013challenges}, Oxford-IIIT Pet \citep{parkhi2012cats}, STL-10 \citep{coates2011analysis}, CIFAR-100 and CIFAR-10 \citep{krizhevsky2009learning}, Food-101 \citep{bossard2014food}, Fashion-MNIST \citep{xiao2017fashion}, EMNIST \citep{cohen2017emnist}, KMNIST \citep{clanuwat2018deep}, and Rendered SST-2 \citep{socher2013recursive}. 
These open-source datasets and fine-tuned checkpoints are obtained from Hugging Face \citep{wolf2019huggingface}. During fine-tuning, the visual encoder is updated while the text encoder remains fixed. Training follows a consistent setup: cross-entropy loss, Adam optimizer, cosine annealing schedule, learning rate of $1 \times 10^{-5}$, batch size of 128, and 4000 training steps. 

\subsection{Task Grouping and Sequence Protocols}
To comprehensively evaluate the scalability and long-horizon stability of continual merging methods under varying workloads, following the protocols in \citep{tang2025merging,qiu2025mingle}, we partition the 20 benchmark datasets into three progressive task groups: 
\begin{itemize}
    \item Short-Range Benchmark (8-Task Group): (1) SUN397, (2) Stanford Cars, (3) RESISC45, (4) EuroSAT, (5) SVHN, (6) GTSRB, (7) MNIST, (8) DTD.
    \item Medium-Range Benchmark (14-Task Group): (1) SUN397, (2) Stanford Cars, (3) RESISC45, (4) EuroSAT, (5) SVHN, (6) GTSRB, (7) MNIST, (8) DTD, (9) Flowers102, (10) PCAM, (11) FER2013, (12) OxfordIIITPet, (13) STL10, (14) CIFAR100.
    \item Long-Range Benchmark (20-Task Group): (1) SUN397, (2) Stanford Cars, (3) RESISC45, (4) EuroSAT, (5) SVHN, (6) GTSRB, (7) MNIST, (8) DTD, (9) Flowers102, (10) PCAM, (11) FER2013, (12) OxfordIIITPet, (13) STL10, (14) CIFAR100, (15) CIFAR10, (16) Food101, (17) FashionMNIST, (18) EMNIST, (19) KMNIST, (20) RenderedSST2.
\end{itemize}
To mitigate the impact of task order and ensure statistical reliability, we conduct 10 independent runs for each task group, each with a distinct randomly permuted task sequence generated with seeds ranging from 42 to 51, as detailed in Table~\ref{tab:task_sequences}. Consequently, we report the mean and standard deviation of both average accuracy (ACC) and backward transfer (BWT) metrics across these trials, providing a robust and consistent evaluation of continual merging performance.
\begin{table*}[h]
    \centering
    \caption{Task sequences used for robustness evaluation. Task IDs correspond to the index in the respective dataset list.}
    \label{tab:task_sequences}
    \resizebox{\textwidth}{!}{
    \begin{tabular}{c|c|l}
    \toprule
    \textbf{Group} & \textbf{Order} & \textbf{Dataset Order (by Task ID)} \\
    \midrule
    \multirow{10}{*}{\rotatebox{90}{\textbf{8 Tasks}}} 
    & 1 & $04\to05\to07\to08\to03\to06\to01\to02$ \\
    & 2 & $07\to08\to05\to04\to02\to06\to03\to01$ \\
    & 3 & $03\to06\to04\to02\to01\to08\to05\to07$ \\
    & 4 & $06\to08\to02\to01\to03\to07\to04\to05$ \\
    & 5 & $07\to06\to03\to08\to05\to01\to04\to02$ \\
    & 6 & $07\to02\to03\to08\to05\to04\to01\to06$ \\
    & 7 & $07\to01\to04\to03\to08\to05\to02\to06$ \\
    & 8 & $08\to05\to06\to07\to01\to04\to03\to02$ \\
    & 9 & $01\to04\to05\to02\to06\to03\to07\to08$ \\
    & 10 & $08\to03\to01\to02\to06\to05\to07\to04$ \\
    \midrule
    \multirow{10}{*}{\rotatebox{90}{\textbf{14 Tasks}}} 
    & 1 & $09\to13\to08\to07\to14\to12\to06\to03\to10\to04\to05\to01\to02\to11$ \\
    & 2 & $09\to10\to11\to14\to07\to13\to04\to02\to06\to08\to03\to12\to05\to01$ \\
    & 3 & $05\to08\to12\to06\to11\to01\to10\to04\to14\to03\to02\to13\to09\to07$ \\
    & 4 & $03\to10\to09\to12\to04\to13\to01\to06\to11\to02\to14\to08\to07\to05$ \\
    & 5 & $08\to14\to09\to06\to12\to13\to05\to03\to04\to11\to10\to01\to07\to02$ \\
    & 6 & $03\to12\to13\to01\to11\to04\to10\to05\to14\to08\to09\to07\to02\to06$ \\
    & 7 & $07\to01\to12\to10\to02\to08\to13\to04\to05\to11\to14\to03\to06\to09$ \\
    & 8 & $05\to12\to04\to11\to03\to08\to10\to01\to09\to13\to14\to07\to06\to02$ \\
    & 9 & $10\to07\to09\to02\to03\to13\to01\to12\to14\to04\to11\to06\to05\to08$ \\
    & 10 & $01\to02\to11\to06\to08\to12\to07\to05\to10\to14\to03\to13\to09\to04$ \\
    \midrule
    \multirow{10}{*}{\rotatebox{90}{\textbf{20 Tasks}}} 
    & 1 & $20\to06\to15\to05\to10\to14\to16\to19\to07\to13\to18\to11\to02\to12\to03\to17\to08\to09\to01\to04$ \\
    & 2 & $09\to14\to06\to03\to07\to04\to18\to01\to17\to19\to08\to20\to13\to16\to11\to12\to15\to05\to10\to02$ \\
    & 3 & $09\to15\to16\to11\to03\to13\to08\to10\to12\to02\to20\to01\to05\to19\to07\to06\to04\to18\to17\to14$ \\
    & 4 & $17\to04\to11\to19\to18\to10\to07\to15\to12\to13\to08\to02\to01\to06\to05\to03\to20\to16\to14\to09$ \\
    & 5 & $14\to16\to04\to20\to15\to17\to07\to11\to06\to18\to12\to01\to19\to09\to10\to05\to08\to02\to13\to03$ \\
    & 6 & $02\to06\to17\to04\to19\to18\to08\to16\to20\to01\to10\to13\to07\to09\to05\to11\to15\to14\to03\to12$ \\
    & 7 & $19\to01\to09\to14\to06\to20\to17\to04\to08\to02\to15\to03\to16\to13\to12\to07\to10\to05\to11\to18$ \\
    & 8 & $15\to07\to08\to02\to10\to06\to17\to20\to05\to19\to16\to01\to18\to09\to13\to11\to04\to14\to12\to03$ \\
    & 9 & $10\to05\to07\to11\to01\to03\to17\to15\to18\to04\to14\to19\to02\to06\to13\to20\to08\to12\to09\to16$ \\
    & 10 & $01\to11\to02\to15\to03\to10\to12\to19\to16\to13\to07\to05\to09\to04\to14\to20\to06\to18\to17\to08$ \\
    \bottomrule
    \end{tabular}%
    }
\end{table*}

\subsection{Baseline Details}
We provide an elaborate description of the three baseline categories to facilitate a clear comparison.

\paragraph{Non-Merging Methods.}
\begin{itemize}
    \item Pre-trained models are used as zero-shot baselines, establishing the lower bound of performance without task-specific adaptation.
    \item Fine-tuned models are independently adapted for each task, serving as specialized references that preserve optimal task-specific performance.
    \item C. Fine-Tuned refers to a standard sequential learning baseline where a single model is iteratively fine-tuned on arriving tasks. 
\end{itemize}

\paragraph{Continual Adaptations of Conventional Merging Methods.}
\begin{itemize}
    \item Simple Weight Averaging (SWA) \citep{izmailov2018averaging} maintains a running average of model parameters to stabilize optimization. At each step, the current merged model is updated by averaging its previous state with the new model parameters. 
    \item Continual Task Arithmetic (C. TA) \citep{ilharco2023editing} treats task-specific updates as vectors and linearly accumulates them onto the pre-trained backbone using a global scaling coefficient. 
    \item Continual Ties-Merging (C. Ties-Merging) \citep{yadav2023ties} prunes redundant parameters and resolves sign conflicts to mitigate task interference, and merges the pruned task vector into the accumulated backbone.
    \item C. LW AdaMerging \citep{yangadamerging} advances Task Arithmetic by learning layer-wise scaling coefficients.
    \item Continual LoRA Weight-Ensembling MoE (C. LoRA-WEMoE) \citep{shen2026efficient} refers to an adaptation of the WEMoE method for parameter-efficient continual merging. Following the setting in \citep{qiu2025mingle}, we compress the MLP experts using LoRA and aggregate them via a gating function.
\end{itemize}

\paragraph{Continual Merging Methods.}
\begin{itemize}
    \item Orthogonal Projection-based Continual Merging (OPCM) \citep{tang2025merging} projects each task vector $\tau^{(t)}$ onto the orthogonal complement of previous updates using a projection mapping $\mathcal{P}_{\alpha}^{(t-1)}$, then generates the new merged model: $\theta_{\mathrm{merged}}^{(t)}=\theta^{(0)} + \frac{\lambda^{(t-1)}\tau_{\mathrm{merged}}^{(t-1)}+\mathcal{P}_{\alpha}^{(t-1)}\tau^{(t)}}{\lambda^{(t)}}$, where $\tau_{\mathrm{merged}}^{(t-1)}=\theta_{\mathrm{merged}}^{(t-1)}-\theta^{(0)}$ and $\lambda^{(t)}$ denotes a time-varying scaling factor. 
    \item Mixture of Null-Space Gated Low-Rank Experts (MINGLE) \citep{qiu2025mingle} represents an MoE-based CMM approach that assigns each incoming model a dedicated expert $f_i$, built upon the null-space projection mechanism of OPCM. It further adapts the associated task-specific gate function $g_i$ at test time using unlabeled test data. The resulting merged model output for a given input $X$ is formulated as: $\theta^{(t)}_{\mathrm{merged}}(X)=\theta^{(0)}(X)+\sum_{i=1}^t{g_i(X)\cdot f_i(X)}$.
\end{itemize}

\section{Comprehensive Results}
\label{ap:results}
\subsection{Detailed Per-Task Performance}
We present the per-task average accuracy across 10 randomized task sequences for the 20-task continual merging experiments in Table~\ref{tab:pertask32}, Table~\ref{tab:pertask16}, and Table~\ref{tab:pertask14}. The overall results of these per-task accuracy are summarized in Table~\ref{tab:main_results} of the main text. Across all three CLIP-ViT backbones, MADE-IT consistently improves performance on the majority of tasks. These per-task breakdowns further substantiate the overall effectiveness of MADE-IT, reinforcing the main findings reported in the paper.

\begin{table*}[htbp]
\tiny
\setlength{\tabcolsep}{2.5pt}
\caption{Per-task results of merging 20 ViT-B/32 models.}
\label{tab:pertask32}
\begin{tabular}{ccccccccccc}
\hline
\textbf{Method}         & \textbf{SUN397}      & \textbf{Cars}          & \textbf{RESISC45}    & \textbf{EuroSAT}     & \textbf{SVHN}        & \textbf{GTSRB}       & \textbf{MNIST}        & \textbf{DTD}         & \textbf{Flowers102}  & \textbf{PCAM}         \\
C. Fine-Tuned           & 53.9                 & 38.2                   & 64.7                 & \textbf{98.7}        & 45.4                 & 34.4                 & 86.7                  & 58.4                 & 57.5                 & 67.7                  \\
Average (SWA)           & 64.2                 & 59.6                   & 64.8                 & 60.9                 & 47.3                 & 43.1                 & 71.8                  & 46.4                 & 66.5                 & 63.9                  \\
C. Task Arithmetic      & 62.0                 & 53.7                   & 60.9                 & 58.1                 & 48.5                 & 48.9                 & 79.4                  & 46.1                 & 61.1                 & 73.4                  \\
C. Ties-Merging         & 62.5                 & 49.1                   & 55.8                 & 50.9                 & 54.6                 & 49.3                 & 82.0                  & 46.7                 & 58.5                 & 69.9                  \\
C. LW AdaMerging        & 63.1                 & 60.0                   & 63.5                 & 60.1                 & 35.6                 & 32.1                 & 51.8                  & 45.4                 & 66.6                 & 60.2                  \\
C. LoRA-WEMoE           & 51.4                 & 45.8                   & 63.3                 & 43.5                 & 42.9                 & 34.6                 & 58.9                  & 46.5                 & 47.5                 & 60.1                  \\
OPCM                    & 64.4                 & 51.1                   & 66.0                 & 71.7                 & 66.1                 & 56.0                 & 90.2                  & 40.4                 & 64.9                 & 80.2                  \\
MINGLE                  & 67.8                 & 58.3                   & 83.5                 & 90.0                 & \textbf{82.9}        & 91.8                 & \textbf{98.0}         & 65.3                 & 74.0                 & 66.9                  \\
\textbf{MADE-IT (Ours)} & \textbf{75.5}        & \textbf{77.6}          & \textbf{86.6}        & 86.5                 & 78.0                 & \textbf{96.3}        & 80.8                  & \textbf{75.1}        & \textbf{84.5}        & \textbf{84.9}         \\ \hline
\multicolumn{1}{l}{}    & \multicolumn{1}{l}{} & \multicolumn{1}{l}{}   & \multicolumn{1}{l}{} & \multicolumn{1}{l}{} & \multicolumn{1}{l}{} & \multicolumn{1}{l}{} & \multicolumn{1}{l}{}  & \multicolumn{1}{l}{} & \multicolumn{1}{l}{} & \multicolumn{1}{l}{}  \\
\textbf{Method}         & \textbf{FER2013}     & \textbf{OxfordIIITPet} & \textbf{STL10}       & \textbf{CIFAR100}    & \textbf{CIFAR10}     & \textbf{Food101}     & \textbf{FashionMNIST} & \textbf{EMNIST}      & \textbf{KMNIST}      & \textbf{RenderedSST2} \\
C. Fine-Tuned           & 58.3                 & 68.5                   & 86.7                 & 40.2                 & 70.5                 & 50.0                 & \textbf{90.7}         & \textbf{72.4}        & 54.5                 & 54.5                  \\
Average (SWA)           & 50.2                 & 84.1                   & \textbf{97.0}        & 69.8                 & 92.7                 & 80.4                 & 71.3                  & 15.0                 & 11.5                 & 61.8                  \\
C. Task Arithmetic      & 51.4                 & 82.3                   & 94.9                 & 64.6                 & 91.4                 & 71.9                 & 73.9                  & 17.8                 & 12.2                 & 59.9                  \\
C. Ties-Merging         & 49.5                 & 81.3                   & 95.2                 & 63.7                 & 91.2                 & 70.2                 & 73.7                  & 17.8                 & 16.9                 & 59.8                  \\
C. LW AdaMerging        & 43.2                 & 83.7                   & 96.8                 & 67.0                 & 89.9                 & 81.6                 & 63.7                  & 16.8                 & 10.7                 & 59.1                  \\
C. LoRA-WEMoE           & 44.6                 & 72.5                   & 86.1                 & 40.1                 & 63.8                 & 63.8                 & 48.1                  & 10.3                 & 12.8                 & 55.7                  \\
OPCM                    & 58.5                 & 82.9                   & 95.9                 & 67.6                 & 92.8                 & 74.0                 & 76.3                  & 22.4                 & 18.3                 & 64.6                  \\
MINGLE                  & 65.0                 & 85.5                   & \textbf{97.0}        & \textbf{72.6}        & \textbf{94.1}        & 81.5                 & 85.4                  & 50.4                 & 65.2                 & 67.1                  \\
\textbf{MADE-IT (Ours)} & \textbf{70.3}        & \textbf{91.8}          & 91.8                 & 71.9                 & 89.7                 & \textbf{88.3}        & 84.8                  & 71.0                 & \textbf{71.4}        & \textbf{71.3}         \\ \hline
\end{tabular}
\end{table*}

\begin{table*}[htbp]
\tiny
\setlength{\tabcolsep}{2.5pt}
\caption{Per-task results of merging 20 ViT-B/16 models.}
\label{tab:pertask16}
\begin{tabular}{ccccccccccc}
\hline
\textbf{Method}         & \textbf{SUN397}  & \textbf{Cars}          & \textbf{RESISC45} & \textbf{EuroSAT}  & \textbf{SVHN}    & \textbf{GTSRB}   & \textbf{MNIST}        & \textbf{DTD}    & \textbf{Flowers102} & \textbf{PCAM}         \\
C. Fine-Tuned           & 62.7             & 58.0                   & 67.6              & \textbf{99.1}     & 46.0             & 29.2             & 93.9                  & 61.9            & 64.1                & 75.2                  \\
Average (SWA)           & 67.1             & 64.6                   & 69.3              & 63.4              & 62.4             & 52.7             & 80.7                  & 46.6            & 71.8                & 63.1                  \\
C. Task Arithmetic      & 65.8             & 57.5                   & 63.8              & 59.5              & 64.7             & 54.0             & 88.0                  & 45.3            & 67.5                & 67.1                  \\
C. Ties-Merging         & 64.2             & 52.9                   & 60.9              & 53.0              & 62.8             & 48.8             & 88.4                  & 45.0            & 61.3                & 68.5                  \\
C. LW AdaMerging        & 65.5             & 65.7                   & 69.8              & 59.4              & 50.1             & 44.2             & 61.1                  & 47.1            & 71.8                & 57.9                  \\
C. LoRA-WEMoE           & 62.7             & 60.2                   & 69.4              & 37.7              & 52.1             & 39.9             & 63.1                  & 45.3            & 64.3                & 51.7                  \\
OPCM                    & 67.9             & 55.9                   & 73.7              & 77.5              & 74.4             & 63.2             & 94.1                  & 49.2            & 72.3                & 79.6                  \\
MINGLE                  & 71.5             & 64.9                   & 85.3              & 90.0              & \textbf{87.5}    & \textbf{90.1}    & \textbf{97.1}         & 62.7            & 82.6                & 80.6                  \\
\textbf{MADE-IT (Ours)} & \textbf{77.4}    & \textbf{85.2}          & \textbf{91.8}     & 95.2              & 67.8             & 88.0             & 89.6                  & \textbf{76.1}   & \textbf{93.5}       & \textbf{88.6}         \\ \hline
                        &                  &                        &                   &                   &                  &                  &                       &                 &                     &                       \\
\textbf{Method}         & \textbf{FER2013} & \textbf{OxfordIIITPet} & \textbf{STL10}    & \textbf{CIFAR100} & \textbf{CIFAR10} & \textbf{Food101} & \textbf{FashionMNIST} & \textbf{EMNIST} & \textbf{KMNIST}     & \textbf{RenderedSST2} \\
C. Fine-Tuned           & 60.5             & 84.5                   & 90.5              & 38.8              & 73.6             & 61.9             & \textbf{89.7}         & \textbf{83.3}   & 51.5                & \textbf{72.8}         \\
Average (SWA)           & 50.9             & 89.6                   & 98.0     & 72.9              & 94.2             & 85.9             & 73.3                  & 15.6            & 12.4                & 62.5                  \\
C. Task Arithmetic      & 50.7             & 89.3                   & 97.0              & 68.0              & 93.1             & 80.3             & 75.7                  & 18.1            & 16.7                & 61.8                  \\
C. Ties-Merging         & 50.4             & 87.9                   & 96.3              & 63.1              & 91.7             & 78.0             & 75.0                  & 23.4            & 24.9                & 61.5                  \\
C. LW AdaMerging        & 46.8             & 88.9                   & \textbf{98.1}     & 69.2              & 91.4             & 86.6             & 67.2                  & 17.2            & 11.0                & 59.2                  \\
C. LoRA-WEMoE           & 45.6             & 91.2                   & 92.3              & 41.3              & 64.3             & 78.1             & 48.0                  & 23.5            & 16.6                & 52.7                  \\
OPCM                    & 59.5             & 91.8                   & 97.7              & 73.2              & 94.7             & 83.1             & 81.3                  & 26.5            & 23.4                & 66.8                  \\
MINGLE                  & 67.6             & 92.7                   & 97.4              & 74.0              & \textbf{95.3}    & 87.7             & 87.4                  & 73.5            & 79.9                & 74.0                  \\
\textbf{MADE-IT (Ours)} & \textbf{72.3}    & \textbf{94.7}          & 92.7              & \textbf{78.1}     & 89.3             & \textbf{91.6}    & 84.4                  & 62.1            & \textbf{85.2}       & 57.2          \\ \hline       
\end{tabular}
\end{table*}

\begin{table*}[htbp]
\tiny
\setlength{\tabcolsep}{2.5pt}
\caption{Per-task results of merging 20 ViT-L/14 models.}
\label{tab:pertask14}
\begin{tabular}{ccccccccccc}
\hline
\textbf{Method}         & \textbf{SUN397}  & \textbf{Cars}          & \textbf{RESISC45} & \textbf{EuroSAT}  & \textbf{SVHN}    & \textbf{GTSRB}   & \textbf{MNIST}        & \textbf{DTD}    & \textbf{Flowers102} & \textbf{PCAM}         \\
C. Fine-Tuned           & 69.5             & 73.6                   & 78.3              & \textbf{99.2}     & 59.3             & 49.3             & \textbf{98.6}         & 69.7            & 83.2                & 78.3                  \\
Average (SWA)           & 70.7             & 77.7                   & 76.4              & 75.3              & 69.5             & 62.1             & 93.7                  & 57.7            & 80.0                & 73.6                  \\
C. Task Arithmetic      & 70.4             & 74.1                   & 73.9              & 66.3              & 69.9             & 65.6             & 95.1                  & 56.6            & 78.6                & 70.4                  \\
C. Ties-Merging         & 69.7             & 70.3                   & 65.3              & 47.9              & 76.1             & 63.6             & 94.7                  & 54.4            & 77.9                & 72.3                  \\
C. LW AdaMerging        & 68.8             & 78.6                   & 75.9              & 65.7              & 58.3             & 51.6             & 79.9                  & 57.4            & 80.6                & 52.4                  \\
C. LoRA-WEMoE           & 62.1             & 68.1                   & 68.7              & 53.2              & 47.5             & 49.4             & 69.8                  & 49.1            & 66.2                & 54.2                  \\
OPCM                    & 73.1             & 78.3                   & 82.4              & 80.2              & 80.8             & 80.4             & 97.4                  & 61.6            & 84.8                & 76.3                  \\
MINGLE                  & 75.9             & 83.4                   & 87.8              & 88.7              & 91.1             & 94.5             & 98.4                  & 70.8            & 94.8                & 75.3                  \\
\textbf{MADE-IT (Ours)} & \textbf{82.6}    & \textbf{92.7}          & \textbf{90.6}     & 90.2              & \textbf{96.3}    & \textbf{98.5}    & 95.4                  & \textbf{82.2}   & \textbf{98.2}       & \textbf{86.4}         \\ \hline
                        &                  &                        &                   &                   &                  &                  &                       &                 &                     &                       \\
\textbf{Method}         & \textbf{FER2013} & \textbf{OxfordIIITPet} & \textbf{STL10}    & \textbf{CIFAR100} & \textbf{CIFAR10} & \textbf{Food101} & \textbf{FashionMNIST} & \textbf{EMNIST} & \textbf{KMNIST}     & \textbf{RenderedSST2} \\
C. Fine-Tuned           & 68.0             & 92.1                   & 94.5              & 60.5              & 85.7             & 74.8             & \textbf{93.1}         & \textbf{89.0}   & 59.2                & 78.8                  \\
Average (SWA)           & 52.7             & 94.2                   & 99.2              & 81.7              & 97.0             & 90.7             & 77.4                  & 16.1            & 10.4                & 66.1                  \\
C. Task Arithmetic      & 55.7             & 94.2                   & 98.6              & 79.1              & 96.6             & 87.6             & 80.8                  & 17.6            & 10.6                & 63.6                  \\
C. Ties-Merging         & 57.6             & 93.5                   & 97.8              & 74.0              & 95.6             & 84.7             & 79.7                  & 20.2            & 12.6                & 58.4                  \\
C. LW AdaMerging        & 49.2             & 93.5                   & \textbf{99.3}     & 77.2              & 95.8             & 91.1             & 68.2                  & 18.6            & 9.8                 & 66.6                  \\
C. LoRA-WEMoE           & 46.3             & 84.5                   & 87.6              & 52.1              & 70.5             & 73.3             & 50.0                  & 18.7            & 10.9                & 56.5                  \\
OPCM                    & 61.8             & 95.4                   & 99.2              & 83.0              & \textbf{97.8}    & 90.9             & 86.0                  & 26.4            & 14.7                & 71.0                  \\
MINGLE                  & 67.7             & \textbf{96.0}          & 98.7              & 81.4              & 97.1             & 90.6             & 90.6                  & 60.7            & \textbf{88.6}       & 79.8                  \\
\textbf{MADE-IT (Ours)} & \textbf{74.7}    & 95.6                   & 98.4              & \textbf{86.9}     & 97.4             & \textbf{94.7}    & 92.9                  & 78.4            & 80.8                & \textbf{79.9}         \\ \hline
\end{tabular}
\end{table*}

\subsection{Additional Results of Expert Evolution}

We provide detailed visualizations of the dynamic expert evolution for the ViT-B/32 and ViT-B/16 architectures in Figure~\ref{apfig:expert_chain}. Consistent with the evolution trajectories analyzed for the ViT-L/14 architecture in the main text, the visualizations across all architectures demonstrate a clear generic-to-specific hierarchical transition. Specifically, in modules located at shallower network depths, the allocation maps are dominated by consistent color blocks, indicating that diverse tasks frequently share a compact set of universal experts. This phenomenon confirms that our manifold-aware strategy effectively identifies high geometric affinity among the principal subspaces of experts in early layers, thereby promoting collaborative sharing. Conversely, as network depth increases, the visualizations exhibit increasingly divergent color variations. This reflects a shift towards expert specialization, where the strategy isolates and retains distinct experts to accommodate highly decoupled, task-specific requirements. Such robust cross-architecture consistency reinforces our conclusion that MADE-IT effectively leverages the intrinsic functional hierarchy of deep neural networks to autonomously manage expert redundancy in complex model streams.

\begin{figure*}[htbp]
\centering
\subfloat[ViT-B/32]{%
    \includegraphics[width=1.0\linewidth]{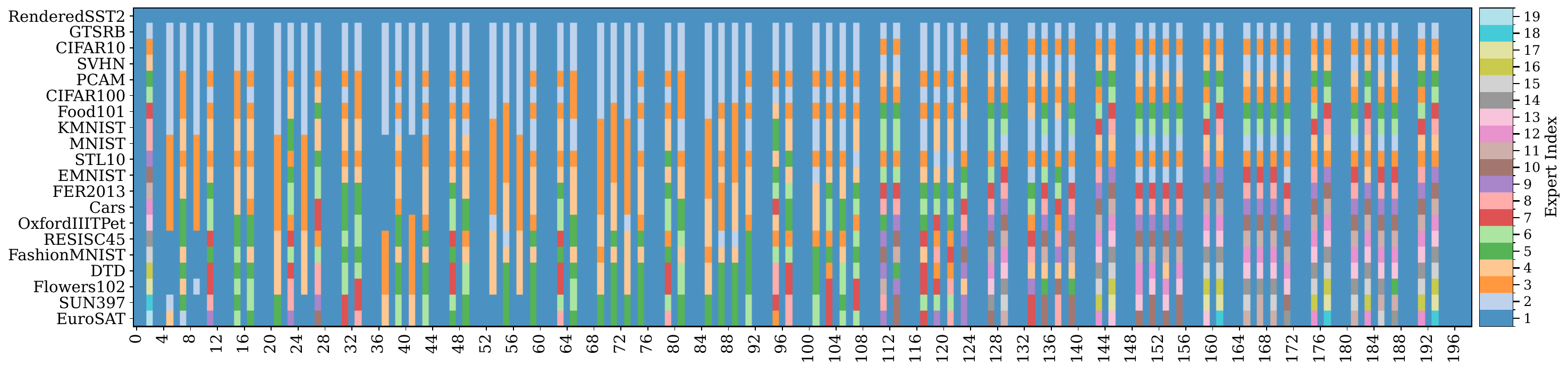}%
    \label{expert_chain:sub1}}\\
\subfloat[ViT-B/16]{%
    \includegraphics[width=1.0\linewidth]{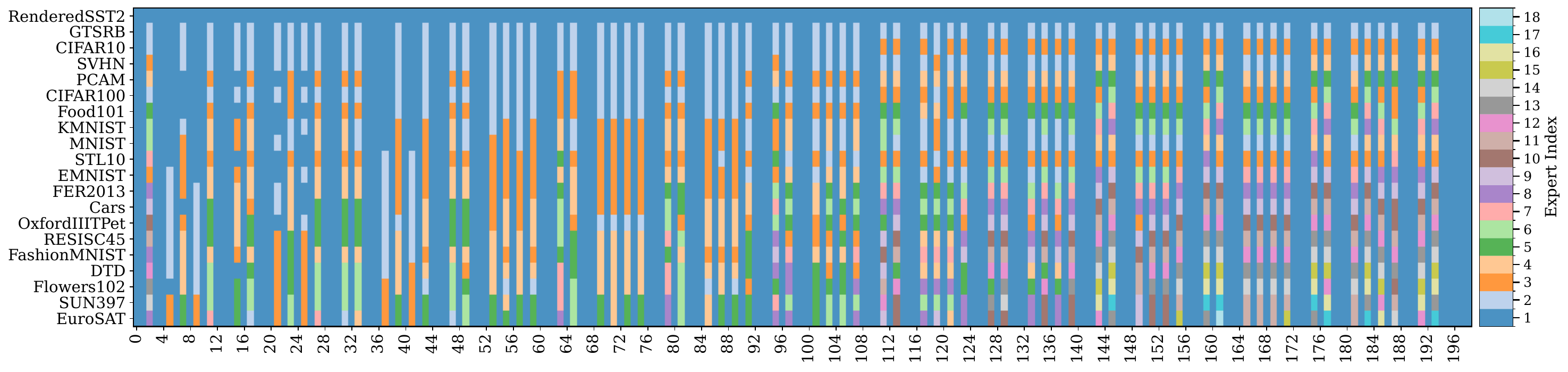}%
    \label{expert_chain:sub2}}\\
\subfloat[ViT-L/14]{%
    \includegraphics[width=1.0\linewidth]{Figures/Expert/ViT-L-14/1expert_chain_all_modules.pdf}%
    \label{expert_chain:sub3}}
\caption{Complete module-wise visualizations of dynamic expert evolution across 20 sequential tasks using the ViT-B/32, ViT-B/16, and ViT-L/14 architectures, respectively. In each visualization, consistent colors denote shared experts among tasks within the corresponding module.}
\label{apfig:expert_chain}
\end{figure*}

Complementing the analysis in the main text, Figure~\ref{ap:expert_analysis} provides the complete quantitative allocation breakdown across all evaluated architectures (ViT-B/32, ViT-B/16, and ViT-L/14). The consistent trends across varying model architectures and scales substantiate the robustness of our dynamic expert evolution strategy. From a component-wise perspective (Figures~\ref{ap:expert_analysis:sub1}-\ref{ap:expert_analysis:sub3}), MLP modules consistently exhibit the highest expert retention with an average of 11 to 12 experts, followed by Attention modules with approximately 8 to 9 experts. This reinforces the critical function of MLP modules in encoding task-specific knowledge. Conversely, other modules maintain significantly high reduction rates, validating their predominantly task-agnostic nature that facilitates extensive expert sharing. From a depth-wise perspective (Figures~\ref{ap:expert_analysis:sub4}-\ref{ap:expert_analysis:sub6}), the number of retained experts progressively increases with network depth, with reduction rates dropping from around 80\% in shallow layers to about 40\% in deeper ones. This validates the premise that earlier layers prioritize capturing shared, generalized representations, whereas deeper layers demand specialized functional pathways to satisfy task-specific requirements.
\begin{figure*}[htbp]
\centering
\subfloat[]{%
    \includegraphics[width=0.33\linewidth]{Figures/Expert/ViT-B-32/1expert_statistics_comparison.pdf}%
    \label{ap:expert_analysis:sub1}}
\hfill
\subfloat[]{%
    \includegraphics[width=0.33\linewidth]{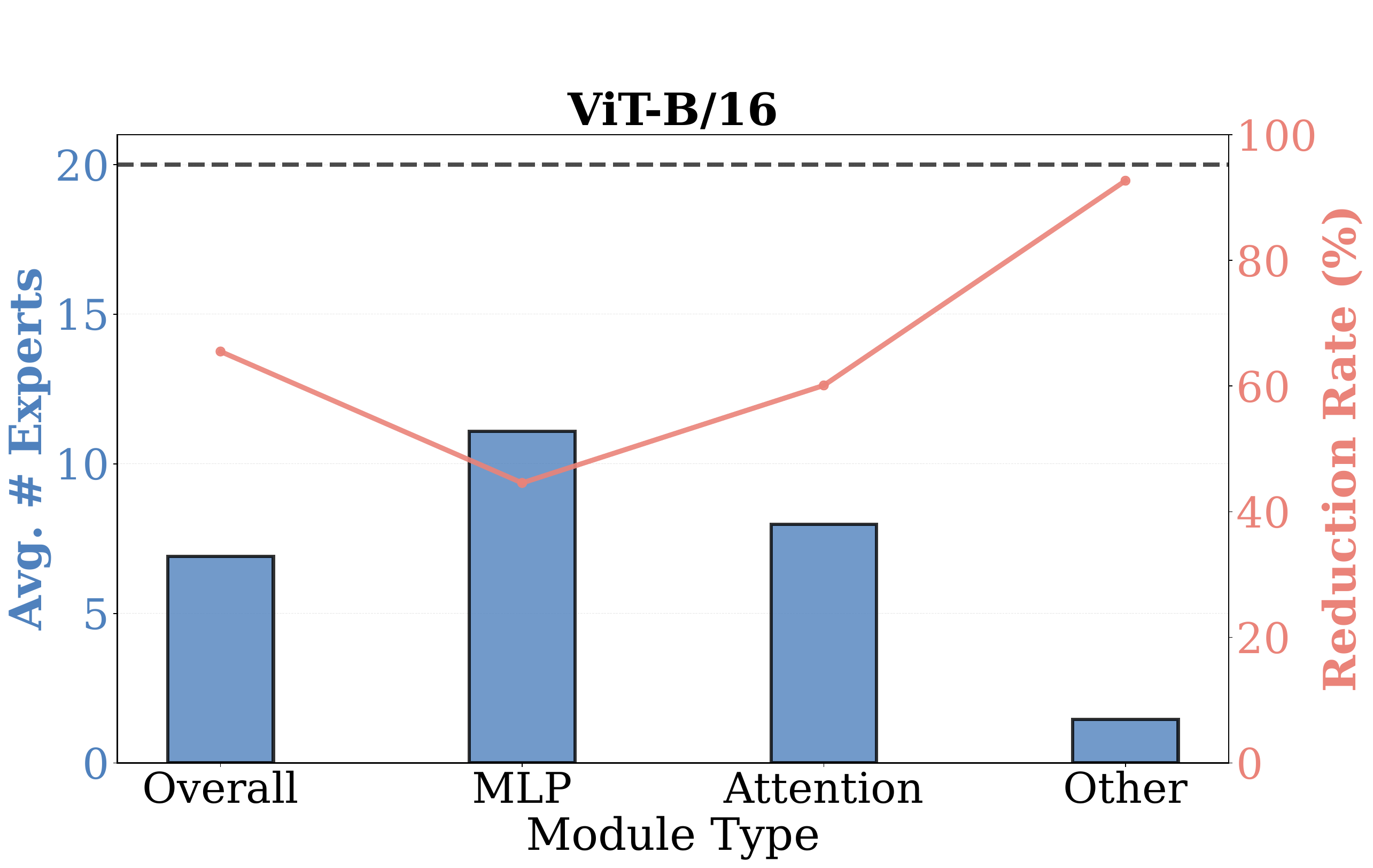}%
    \label{ap:expert_analysis:sub2}}
\hfill
\subfloat[]{%
    \includegraphics[width=0.33\linewidth]{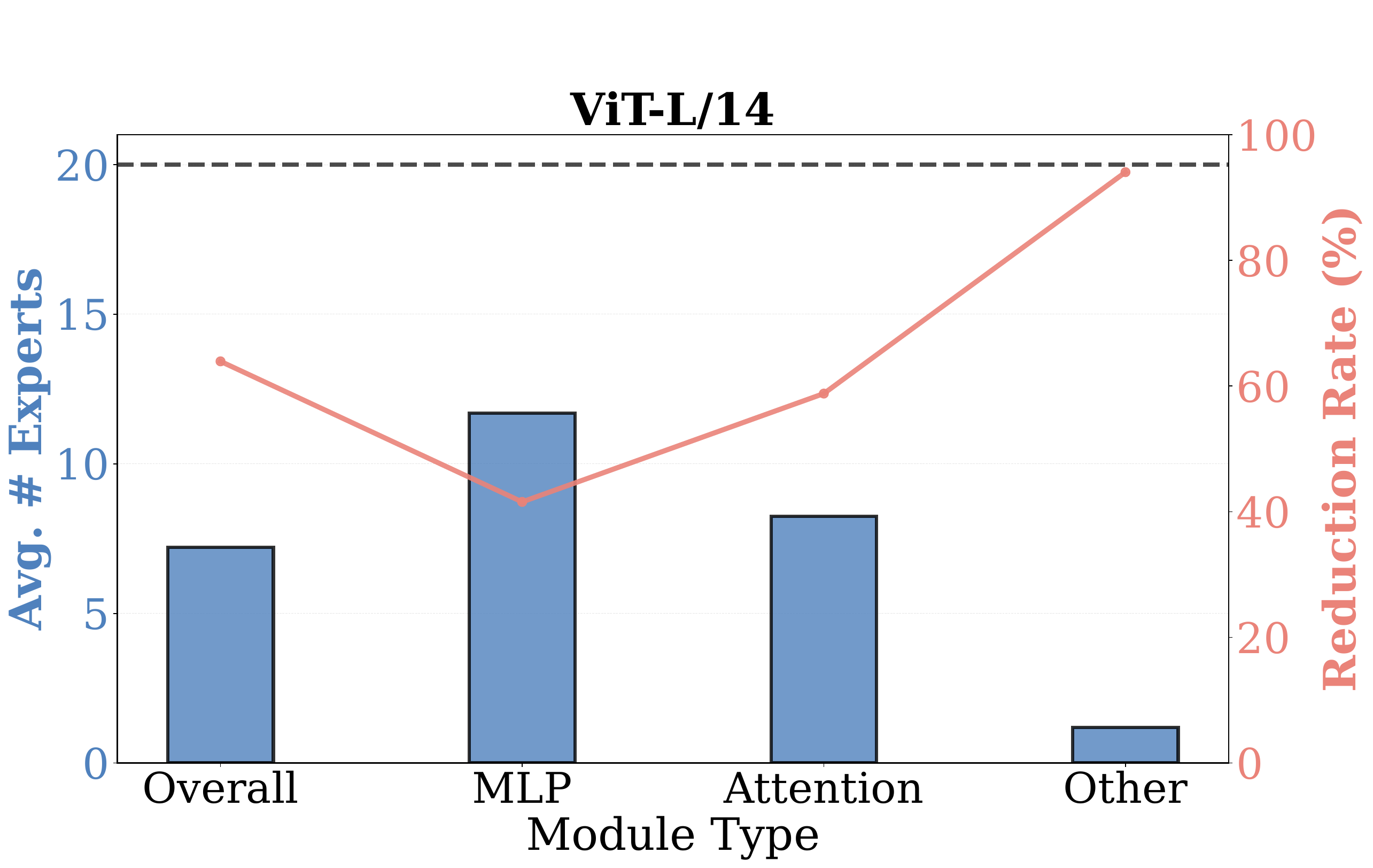}%
    \label{ap:expert_analysis:sub3}}\\[1ex]
\subfloat[]{%
    \includegraphics[width=0.33\linewidth]{Figures/Expert/ViT-B-32/1expert_layer_statistics_comparison.pdf}%
    \label{ap:expert_analysis:sub4}}
\hfill
\subfloat[]{%
    \includegraphics[width=0.33\linewidth]{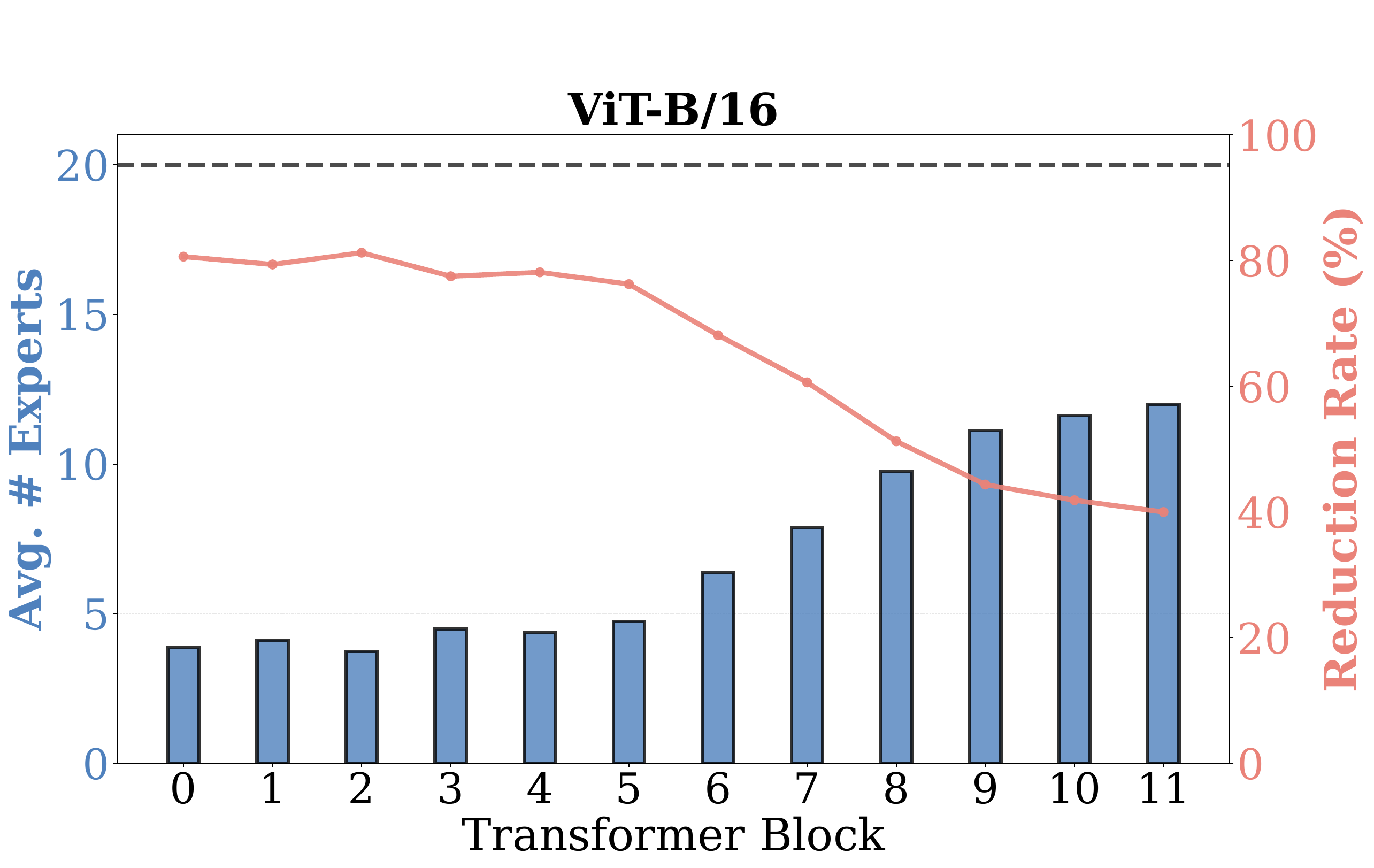}%
    \label{ap:expert_analysis:sub5}}
\hfill
\subfloat[]{%
    \includegraphics[width=0.33\linewidth]{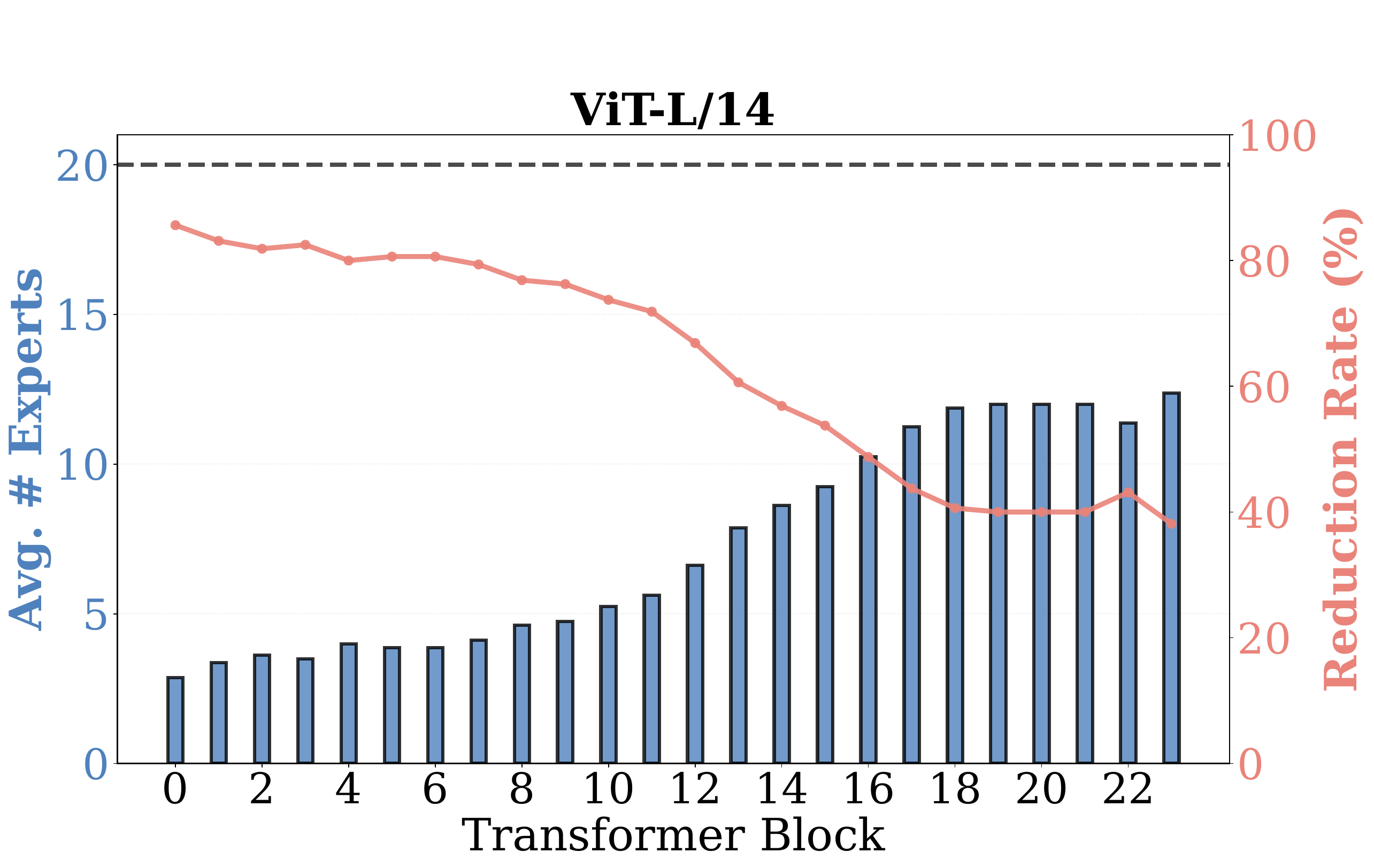}%
    \label{ap:expert_analysis:sub6}}
\caption{Comprehensive quantitative analysis of dynamic expert evolution across components and network depths for all evaluated architectures. The top row (a-c) shows the average number of retained experts and reduction rates by module type (MLP, Attention, Other). The bottom row (d-f) illustrates the layer-wise distribution across Transformer blocks.} 

\label{ap:expert_analysis}
\end{figure*}

\subsection{Additional Results of Density Distributions}
To decipher the mechanism underlying the performance gains of our subspace affinity, we provide comprehensive density distributions of pairwise similarity scores for the ViT-B/32, ViT-B/16, and ViT-L/14 architectures in Figure~\ref{apfig:Metric_Distribution}. Consistent with the ViT-B/32 observations discussed in the main text, the comparative results across different architectures and model scales demonstrate a clear concentration phenomenon when utilizing cosine similarity, where similarity scores cluster tightly around zero. This consistent collapse is fundamentally attributed to the inherent orthogonality of high-dimensional parameter spaces. Such concentration renders parameter-space metrics incapable of distinguishing the subtle differences among task-specific updates. In contrast, our projection-based subspace affinity consistently exhibits a much smoother distribution with a broader dynamic range across all evaluated architectures. This enhanced geometric discriminability provides critical informational signals that enable the effective and robust quantification of intrinsic inter-expert correlations necessary for dynamic expert evolution.

\begin{figure*}[htbp]
\centering
\subfloat[]{%
    \includegraphics[width=0.33\linewidth]{Figures/simi/ViT-B-32/1similarity_histogram.pdf}%
    \label{Metric_Distribution:sub1}}
\hfil
\subfloat[]{%
    \includegraphics[width=0.33\linewidth]{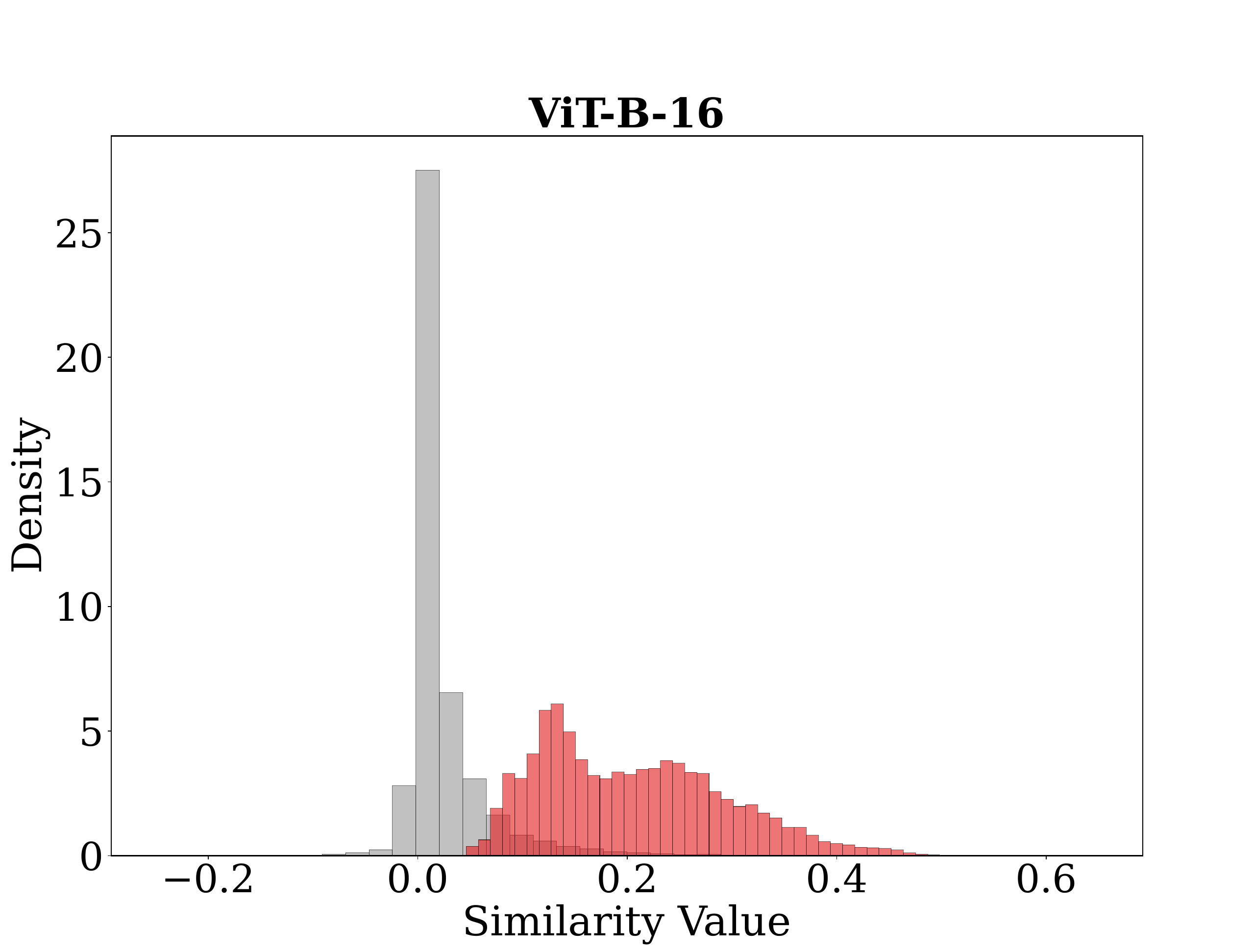}%
    \label{Metric_Distribution:sub2}}
\hfil
\subfloat[]{%
    \includegraphics[width=0.33\linewidth]{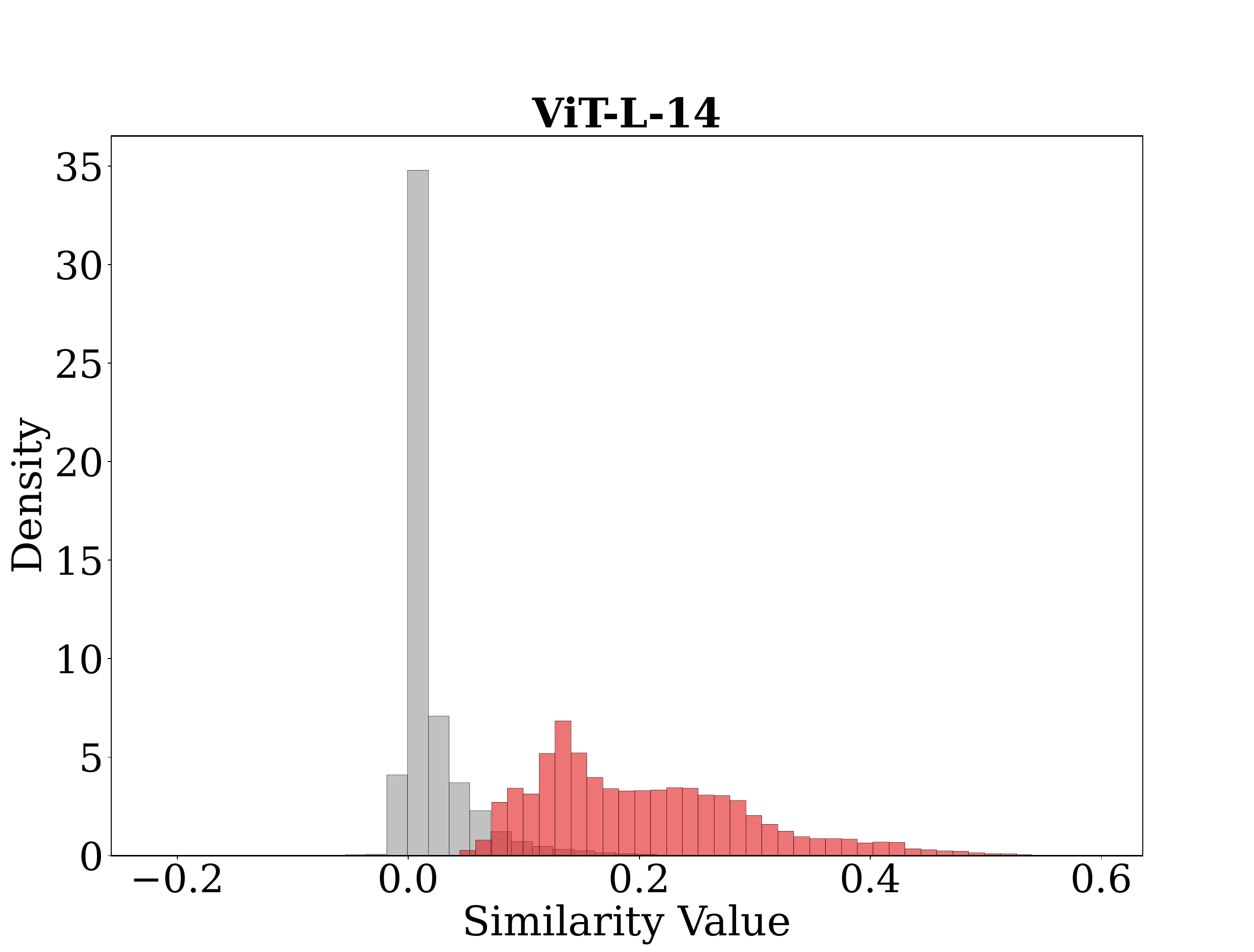}%
    \label{Metric_Distribution:sub3}}
\caption{Density distribution of pairwise similarity scores across 20 task-specific models.}
\label{apfig:Metric_Distribution}
\end{figure*}

\subsection{Additional Results of Hyperparameter Sensitivity}
\paragraph{Sensitivity to Rank Ratio $\rho$:}
We investigate the influence of the truncation rank ratio $\rho \in \{0.1, 0.2, 0.4, 0.6, 0.8, 1.0\}$ across all architectures in Figure~\ref{apfig:hyperR}. Consistent with the ViT-L/14 observations in the main text, both average accuracy and backward transfer curves across the ViT-B/32 and ViT-B/16 architectures exhibit a characteristic performance plateau within the $\rho \leq 0.6$ regime, followed by a sharp degradation as $\rho$ approaches $1.0$. This pattern corroborates the low-rank prior of fine-tuning updates, indicating that a minimal set of principal parameters is sufficient to represent task-specific functionalities with high fidelity. Moreover, the sharp performance degradation at higher ratios highlights the necessity of compact subspaces for our implicit routing mechanism. While higher rank ratios theoretically preserve more information, they invariably introduce substantial task-irrelevant noise. As these expert subspaces gradually expand to approximate the ambient full parameter space, feature projections become increasingly indistinguishable. This degeneration obscures the discriminability required by our implicit routing mechanism, generating misleading routing decisions and exacerbating catastrophic forgetting. Consequently, a compact $\rho=0.1$ not only optimizes parameter efficiency through low-rank truncation but also preserves the geometric specificity of experts via denoising, ensuring precise expert activation and superior multi-task performance.

\begin{figure*}[htbp]
\centering
\subfloat[]{%
    \includegraphics[width=0.33\linewidth]{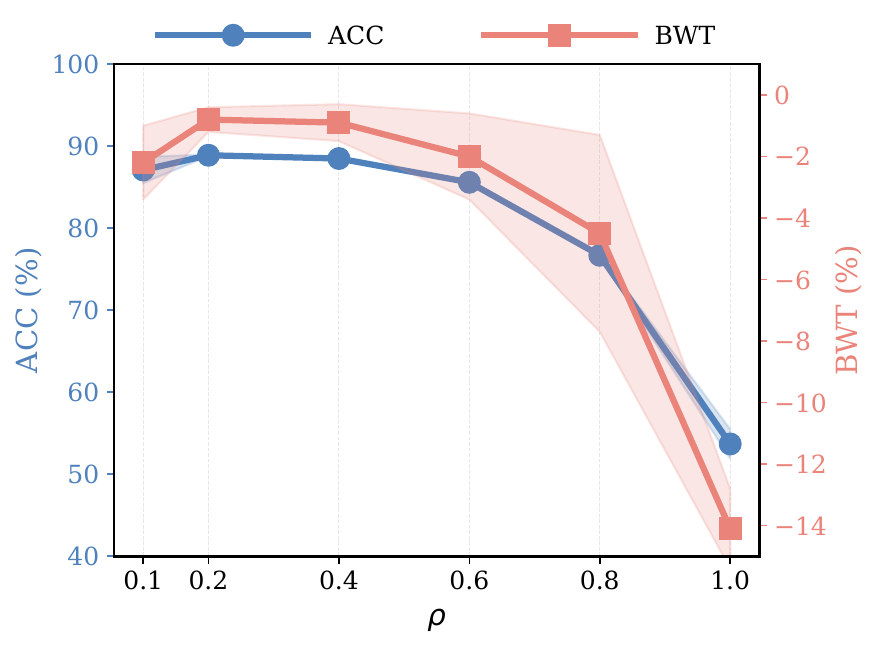}%
    \label{hyperR:sub1}}
\hfil
\subfloat[]{%
    \includegraphics[width=0.33\linewidth]{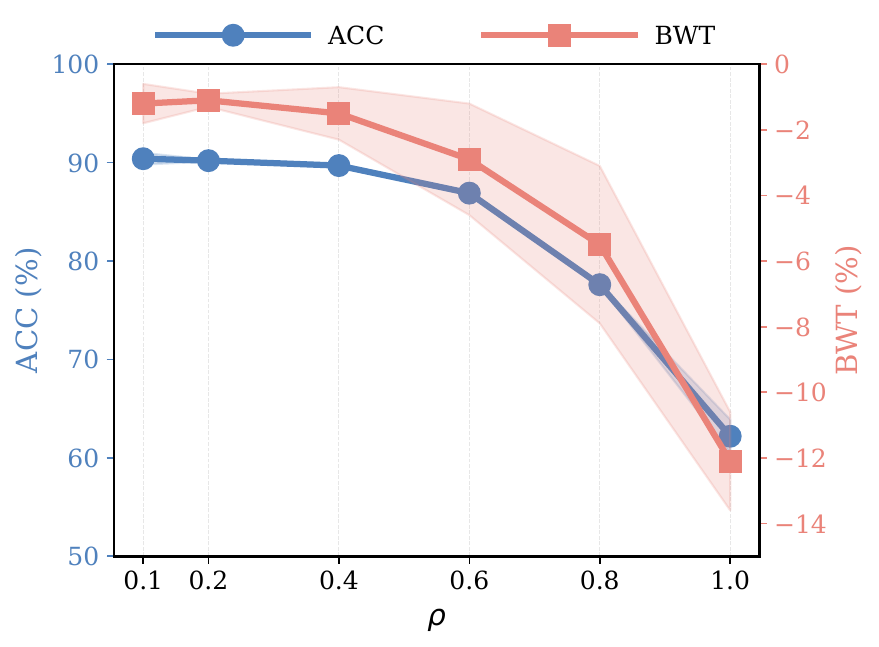}%
    \label{hyperR:sub2}}
\hfil
\subfloat[]{%
    \includegraphics[width=0.33\linewidth]{Figures/hyperR/ViT-L-14_rank_ratio_ablation.pdf}%
    \label{hyperR:sub3}}
\caption{Sensitivity analysis of the rank ratio $\rho$ across the 8-task setting on (a) ViT-B/32, (b) ViT-B/16, and (c) ViT-L/14. Shaded areas denote the standard deviation.}
\label{apfig:hyperR}
\end{figure*}

\paragraph{Sensitivity to Margin Coefficient $\beta$:}
We investigate the influence of the margin coefficient $\beta \in \{-1.5, -1.0, -0.5, 0.0, 0.5, 1.0, 1.5, 2.0\}$ across all architectures in Figure~\ref{apfig:hyperK}. Consistent with the ViT-L/14 observations in the main text, both average accuracy and backward transfer scores across the ViT-B/32 and ViT-B/16 models exhibit a monotonic upward trend as $\beta$ increases, particularly in the regime $\beta < 0.5$. This indicates that raising the adaptive threshold relaxes the creation criterion, forming diverse expert ensembles capable of capturing distinct task-specific features. Notably, a characteristic performance plateau emerges within the interval $\beta \in [1.0, 2.0]$, where further elevating the threshold yields diminishing performance gains while unnecessarily inflating the overall parameter footprint due to redundant experts. Across all architectures, the performance inflection point consistently aligns near $\beta=1.0$. Therefore, we adopt $\beta=1.0$ as the default configuration, as it strikes an optimal balance between task-sensitive representational fidelity and architectural parsimony, ensuring robust performance across varying model scales and task sequences.

\begin{figure*}[htbp]
\centering
\subfloat[]{%
    \includegraphics[width=0.33\linewidth]{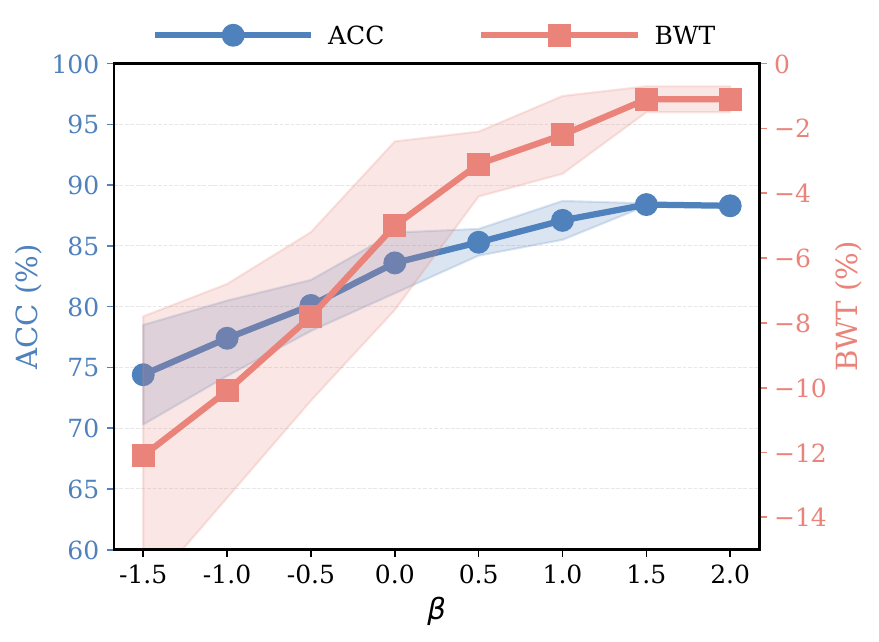}%
    \label{hyperK:sub1}}
\hfil
\subfloat[]{%
    \includegraphics[width=0.33\linewidth]{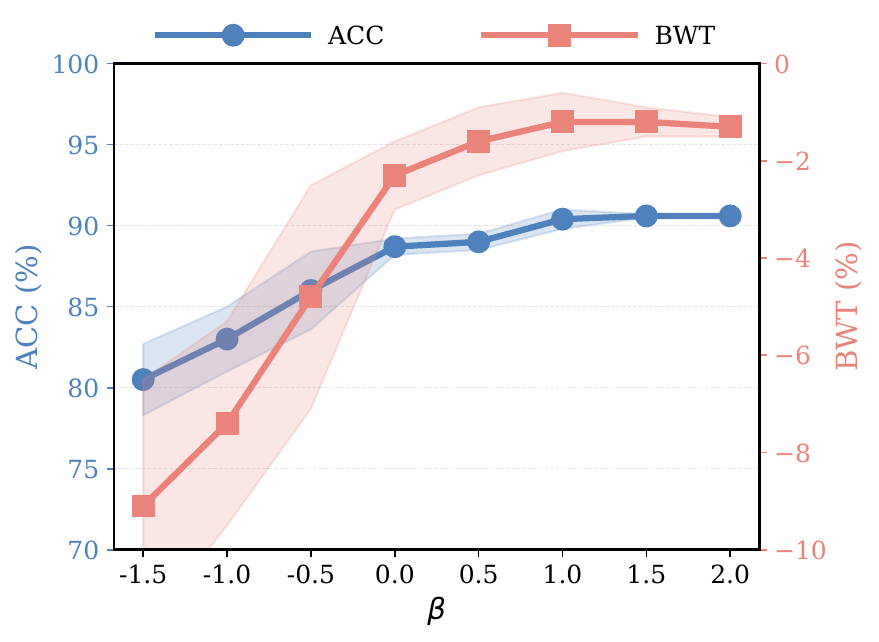}%
    \label{hyperK:sub2}}
\hfil
\subfloat[]{%
    \includegraphics[width=0.33\linewidth]{Figures/hyperK/ViT-L-14_hyperK_ablation.pdf}%
    \label{hyperK:sub3}}
\caption{Sensitivity analysis of the margin coefficient $\beta$ across the 8-task setting on (a) ViT-B/32, (b) ViT-B/16 models, and (c) ViT-L/14. Shaded areas denote the standard deviation.}
\label{apfig:hyperK}
\end{figure*}



\end{document}